\documentclass[11pt]{article}

\usepackage{amssymb,amsmath,amsthm}
\usepackage{bbm}
\usepackage[noend]{algpseudocode}
\usepackage{graphicx}
\usepackage{verbatim}
\usepackage{url,xspace,hyperref}
\usepackage{cite}
\usepackage{caption}
\usepackage{subcaption}
\usepackage{multirow}
\usepackage{multicol}
\usepackage{latexsym}
\usepackage{amsmath,amssymb,enumerate}
\usepackage{algorithm,algorithmicx}
\usepackage{float}
\usepackage{xcolor}
\usepackage{mathrsfs}
\usepackage{cleveref}
\usepackage{booktabs}

\usepackage{fullpage}
\usepackage{geometry}
\usepackage{setspace}
\usepackage{tcolorbox}
\usepackage{xspace}

\tcbuselibrary{skins,breakable}
\tcbset{enhanced jigsaw}

\newtheorem{theorem}{Theorem}[section]
\newtheorem*{theorem*}{Theorem}

\newtheorem{lemma}[theorem]{Lemma}

\newcommand\numberthis{\addtocounter{equation}{1}\tag{\theequation}}

\newcounter{note}[section]

\def\sse{\subseteq}

\newcommand{\pr}{\mathbf{P}} 
\newcommand{\E}{\mathbb{E}}

\newcommand{\I}{\mathcal{I}}

\newcommand{\A}{\mathcal{A}}
\newcommand{\cS}{\mathcal{S}}

\newcommand{\param}{\eta}

\newcommand{\ignore}[1]{}

\newcommand{\paren}[1]{\ensuremath{\left(#1\right)}\xspace}

\newcommand{\half}{\frac{1}{2}}
\newcommand{\B}{\mathcal{B}}

\newcommand{\pcomp}{$\mathtt{PCOMP}$}
\newcommand{\scomp}{$\mathtt{SCOMP}$}
\newcommand{\rscomp}{$\mathtt{R}$-$\mathtt{SCOMP}$}

\newenvironment{tbox}{\begin{tcolorbox}[
		enlarge top by=5pt,
		enlarge bottom by=5pt,
		 breakable,
		 boxsep=0pt,
                  left=4pt,
                  right=4pt,
                  top=10pt,
                  arc=0pt,
                  boxrule=1pt,toprule=1pt,
                  colback=white
                  ]
	}
{\end{tcolorbox}}

\newcommand{\textbox}[2]{
{
\begin{tbox}
\textbf{#1}
{#2}
\end{tbox}
}
}

\newif\ifConfVersion

\title{Batched Dueling Bandits} 
\author{Arpit Agarwal \and Rohan Ghuge \and Viswanath Nagarajan}
\begin{document}

\maketitle

\begin{abstract}
    The $K$-armed dueling bandit problem, where the feedback is in the form of noisy pairwise comparisons, has been widely studied. Previous works have only focused on the sequential setting where the policy adapts after every comparison. However, in many applications such as search ranking and recommendation systems, it is preferable to perform comparisons in a limited number of \emph{parallel  batches}. We study the  \emph{batched $K$-armed dueling bandit} problem under two standard settings: (i) existence of a Condorcet winner, and (ii) strong stochastic transitivity and  stochastic  triangle inequality. For both settings, we obtain algorithms with a smooth trade-off between the number of batches and regret. Our regret bounds  match the best known sequential regret bounds (up to poly-logarithmic factors),  using only a logarithmic number of batches. We complement our regret analysis with a nearly-matching lower bound. Finally, we also validate our theoretical results via experiments on synthetic and real data.
\end{abstract}

\section{Introduction}

The $K$-armed dueling bandits problem has been widely studied in  machine learning due to its applications in search ranking, recommendation systems, sports ranking, etc.\ \cite{YueJo11,YueBK+12,Urvoy+13,Ailon+14,Zoghi+14,Zoghi+15,Zoghi+15a,Dudik+15,Jamieson+15,Komiyama+15a,Komiyama+16,Ramamohan+16,ChenFr17}. 
It is a variation of the traditional stochastic bandit problem in which feedback is obtained in the form of pairwise preferences.
This problem falls under the umbrella of \emph{preference learning} \cite{WirthAN+17},
where the goal is to learn from relative feedback (in our case, given two alternatives, which of the two is preferred). Designing learning algorithms for such relative feedback becomes crucial in domains where qualitative feedback is easily obtained, but real-valued feedback would be arbitrary or not interpretable. We illustrate this using the web-search ranking application.

Web-search ranking is an example of a complex information retrieval system, where the goal is to provide a list (usually \emph{ranked}) of candidate documents to the user of the system in response to a query~\cite{RadlinskiKJ02, Joachims02, YueJoachims09, HofmannWR13}.
Modern day search engines  comprise hundreds of parameters which are used to output a ranked list in response to a query. 
However, manually tuning these parameters can sometimes be infeasible, and online learning frameworks (based on user feedback) have been invaluable in automatically tuning these parameters~\cite{Liu09}. 
These methods do not affect user experience, enable the system to continuously learn about user preferences, and thus continuously adapt to user behavior.
For example, given two rankings $\ell_1$ and $\ell_2$, they can be interleaved and presented to the user in such a way that clicks indicate which of the two rankings is more preferable to the user~\cite{RadlinskiKJ02}.
The availability of such pairwise comparison data 
motivates the study of learning algorithms that exploit such relative feedback.

Previous learning algorithms have focused on a \emph{fully adaptive} setting; in the web-ranking application this corresponds to 
the learning algorithm updating its parameters after each query. Such updates might be impractical in large systems for the following reasons.  If the parameters are fine-tuned for  each user  and  users make multiple queries in a short time, such continuous updates  require a lot of computational power. Even if users are assigned to a small number of classes (and parameters are fine-tuned for each user-class), multiple users from the same class may simultaneously query the system, making it
impractical  to adapt after each interaction.

Motivated by this, we introduce the \emph{batched $K$-armed dueling bandits problem} (or, batched dueling bandits), where the learning algorithm is only allowed to \emph{adapt a limited number of times}.
Specifically, the algorithm uses at most $B$ \emph{adaptive rounds} and 
in each round it 
commits to a fixed \emph{batch} of 
pairwise comparisons.
The feedback 
for a batch is received simultaneously, 
and the algorithm chooses the next batch based on this (and previous) feedback.

We design four algorithms, namely \pcomp, \scomp, \scomp2 and \rscomp, for batched dueling bandits under a finite time-horizon $T$. We analyze the regret of \pcomp \ under the {Condorcet} assumption, and that of 
the others
under the {strong stochastic transitivity} (SST) and {stochastic triangle inequality} (STI) assumptions. In all cases, we obtain a smooth trade-off between the expected regret and the number of batches, $B$. We complement our upper bound with a nearly matching lower bound on the expected regret. Finally, we run computational experiments to validate our theoretical results.

\subsection{Preliminaries}
The \emph{$K$-armed dueling bandits} problem \cite{YueBK+12} is an online optimization problem, where the goal is to find the best among $K$ bandits $\B = \{b_1, \ldots, b_K\}$ using noisy pairwise comparisons with low \emph{regret}.
In the traditional multi-armed bandit problem \cite{AuerCF02}, an \emph{arm} (or equivalently, bandit) $b_j$ can be pulled at each time-step $t$, which generates a random reward from an unknown stationary distribution with expected value $\mu_j$. 
However, in the $K$-armed dueling bandits problem, each iteration comprises a noisy comparison between two bandits (possibly the same), say $(b_i, b_j)$. 
The outcome of the comparison is an independent random variable, and the probability of picking $b_i$ over $b_j$ is a constant denoted $P_{i, j} = \frac{1}{2} + \epsilon_{i, j}$ where $\epsilon_{i, j} \in (-\frac{1}{2}, \frac{1}{2})$. 
Here $\epsilon_{i, j}$ can be thought of as a measure of distinguishability between the two bandits,
and we use $b_i \succ b_j$ when $\epsilon_{i, j} > 0$. We also refer to $\epsilon_{i, j}$ as the \emph{gap} between $b_i$ and $b_j$.

Throughout the paper, we let $b_1$ refer to the best bandit. To further simplify notation, we define $\epsilon_j = \epsilon_{1, j}$; that is, the gap between $b_1$ and $b_j$. We define the \emph{regret} per time-step as follows: suppose bandits $b_{t_1}$ and $b_{t_2}$ are chosen in iteration $t$, then the regret $r(t) = \frac{\epsilon_{t_1} + \epsilon_{t_2}}{2}$. 
The cumulative regret up to time $T$ is $R(T) = \sum_{t=1}^T r(t)$, where $T$ is the time horizon, and it's assumed that $K \leq T$.  
The cumulative regret can be equivalently stated as $R(T) = \frac{1}{2}\sum_{j=1}^{K} T_j \epsilon_{j}$, where $T_j$ denotes the number comparisons involving $b_j$.
We define $\epsilon_{\min} = \min_{j:\epsilon_j>0}\epsilon_j$ to be the smallest non-zero gap of any bandit with $b_1$.
We say that bandit $b_i$ is a \emph{Condorcet winner} if, and only if, $P_{i, j} \geq \frac{1}{2}$ for all $j \in \B \setminus \{i\}$. Furthermore, we say that the probabilistic comparisons exhibit {\em strong stochastic transitivity} (SST) if there exists an ordering, denoted by $\succeq$, over arms such that for every triple $b_i \succeq b_j \succeq b_k$, we have $ \epsilon_{i, k} \geq \max\{\epsilon_{i,j}, \epsilon_{j, k}\}, $
and exhibits  {\em stochastic triangle inequality} (STI)
if  for every triple $b_i \succeq b_j \succeq b_k$,
$\epsilon_{i, k} \leq \epsilon_{i,j} + \epsilon_{j, k}. $
\ifConfVersion
\else
\fi 

\subsection{Batch Policies}
In traditional bandit settings, actions are performed \emph{sequentially}, utilizing  the results of \emph{all prior actions} in determining the next action. In the batched setting, the algorithm must commit to a round (or \emph{batch}) of actions to be performed  \emph{in parallel}, and can only observe the results after all   actions in the batch have been performed. More formally, in round $r = 1, 2, \ldots$, the algorithm must decide the comparisons to be performed; afterwards \emph{all}  outcomes of the comparisons in batch $r$ are received. The algorithm can then, \emph{adaptively}, select the next batch of comparisons. However, it can use at most a given number, $B$, of batches.

The batch sizes can be chosen \emph{non-adaptively} (fixed upfront) or \emph{adaptively}. In an adaptive policy the batch sizes may even depend on previous observations of the algorithm. An adaptive policy is more powerful than a non-adaptive policy, and may suffer a smaller regret. In this paper, we focus on such adaptive policies. Furthermore, note that the total number of comparisons (across all batches) must sum to $T$. We assume that the values of $T$ and $B$ are known. Observe that when $T = B$, we recover the fully sequential setting.

\subsection{Results and Techniques}

\ifConfVersion
\else

\begin{table}
    \centering
    {\renewcommand{\arraystretch}{2}%
    \begin{tabular}{|c|c|cc|c|}
        \hline
         \multirow{2}{*}{Setting} &{Fully Adaptive} & \multicolumn{2}{c|}{\textbf{Our Algorithms}}  & \textbf{Our Lower Bound} \\
        & (prior work) & Regret & Rounds & (for $B$ rounds) \\
        \hline
        Condorcet & $O\left(K\frac{\log{T}}{\epsilon_{\min}}\right) + O\left(\frac{K^2}{\epsilon_{\min}}\right)$ &  $O\left(\frac{ K^2T^{1/B} \log(T)}{\epsilon_{\min}}\right)$ & $B$ &  $\Omega\left(\frac{K T^{1/B}}{B^2\epsilon_{\min}} \right) $\\ 
        \hline
        \multirow{1}{*}{SST + STI} & \multirow{1}{*}{$ O\left(\frac{K\log(T)}{\epsilon_{\min}}\right)$} & $O\left(\frac{ KBT^{1/B} \log(T)}{\epsilon_{\min}}\right)$ & $2B+1$ &  \multirow{1}{*}{$\Omega\left( \frac{K T^{1/B}}{B^2\epsilon_{\min}} \right) $} \\
        \hline
    \end{tabular}
    }
    \caption{A summary of our results}
    \label{table:results}
\end{table}
\fi

We provide a summary of our results in \Cref{table:results}. Our first result is as follows.
\begin{theorem}\label{thm:condorcet-init}
For any integer $B > 1$, there is an algorithm for batched dueling bandits 
that uses at most $B$ rounds, and if the  instance admits a Condorcet winner, 
the expected regret is bounded by
$$ \E[R(T)] \leq 3KT^{1/B} \log\left(6TK^2B\right) \sum_{j: \epsilon_{j} > 0} \frac{1}{\epsilon_{j}}.$$
\end{theorem}

The above bound is an instance-dependent bound. To obtain an instance-independent bound, recall that $\epsilon_{\min} = \min_{j:\epsilon_j > 0}\epsilon_j$. We get that the expected worst-case regret is bounded by 
$$  \E[R(T)] \leq  \frac{ 3K^2T^{1/B} \log\left(6TK^2B\right))}{\epsilon_{\min}}.$$

In the sequential setting, \cite{Zoghi+14, Komiyama+15a} achieve a  bound of $O\left(K\frac{\log{T}}{\epsilon_{\min}}\right) + O\left(\frac{K^2}{\epsilon_{\min}}\right)$ on the expected regret in the worst-case.
When $B = \log(T)$, our worst-case regret is at most $$\E[R(T)] \leq 3K^2\log(6TK^2B)/\epsilon_{\min} = O(K^2\log(T)/\epsilon_{\min}),$$ which nearly matches the best-known bound in the sequential setting.
Our algorithm in Theorem~\ref{thm:condorcet-init} proceeds by performing all pairwise comparisons in an \emph{active set} of bandits, and gradually eliminating sub-optimal bandits. This algorithm is straightforward, and its analysis follows that of~\cite{EsfandiariKM+21} for batched stochastic multi-armed bandits. 
Although this is a simple result, it is an important step  for our main results, described next. 

Our main results are 
when the instance satisfies the SST and STI conditions. 
These conditions impose a structure on the pairwise preference probabilities, and we are able to exploit this additional structure to obtain improved bounds.

\begin{theorem}\label{thm:sst-sti-dep-init}
For any integer $B > 1$, there is an algorithm for batched dueling bandits
that uses at most $B+1$ rounds, and if the instance satisfies the SST and STI 
assumptions,
 the expected regret is bounded by $$ \E[R(T)] = \sum_{j: \epsilon_{j} > 0}O\left(\frac{\sqrt{K}T^{1/B} \log(T)}{\epsilon_{j}}\right). $$ 
\end{theorem}

The idea behind this algorithm is to first sample a ``sufficiently small'' \emph{seed set}, and then to perform all pairwise comparisons between the seed set and the active set to eliminate sub-optimal arms. The idea is to exploit the structure of pairwise probabilities so that we do not need to perform \emph{all} pairwise comparisons.
Additionally, if the seed set is found to be sub-optimal, we 
can construct a \emph{much smaller} active set; thus allowing us to switch to the pairwise comparison policy.
In the sequential setting, \cite{YueBK+12} obtain instance-dependent regret bounded by $\sum_{j: \epsilon_j > 0} O\left(\frac{\log(T)}{\epsilon_j}\right)$. Our result nearly matches this sequential bound (with an extra multiplicative factor of $\sqrt{K}$) when $B = \log(T)$. Observe that the worst-case regret of \cite{YueJo11} in the sequential setting is bounded by $ O\left(\frac{K\log(T)}{\epsilon_{\min}}\right)$, while we obtain  $\E[R(T)] \leq O\left(\frac{K\sqrt{K}T^{1/B}\log(T)}{\epsilon_{\min}}\right)$.

Next, we   improve the worst-case regret by reducing the comparisons performed as follows. We first perform pairwise comparisons amongst bandits in the seed set, and pick a candidate bandit. This candidate bandit is used to eliminate sub-optimal arms from the active set. Although selecting a candidate bandit each time requires additional adaptivity, 
we get a better bound on the worst-case expected regret by exploiting the fact that there can be at most $B$ candidate bandits.

\begin{theorem}\label{thm:sst-sti-indep-init}
For any integer $B > 1$, there is an algorithm for batched dueling bandits
that uses at most $2B+1$ rounds, and if the instance satisfies the SST and STI 
assumptions,
  the expected worst-case regret is bounded by
$$  \E[R(T)] = O\left(\frac{KB T^{1/B} \log(T)}{\epsilon_{\min}}\right). $$ 
\end{theorem}
Thus, in $B = \log(T)$ rounds, our expected worst-case regret is bounded by $E[R(T)] \leq O\left(\frac{K \log^2(T)}{\epsilon_{\min}}\right)$ matching the best known result in the sequential setting up to an additional logarithmic factor.


We also  improve the  instance-dependent regret bound in Theorem~\ref{thm:sst-sti-dep-init} by using a few additional rounds. In particular, using the approach in Theorem~\ref{thm:sst-sti-dep-init} along with recursion, we obtain:
\begin{theorem}\label{thm:sst-sti-general-init}
For any integers $B > 1$, $m \geq 0$ and parameter $\param \in (0, 1)$, there is an algorithm for batched dueling bandits
that uses at most $B+m$ rounds, and if the instance satisfies the SST and STI 
assumptions,
the expected regret is bounded by
$$  \E[R(T)] = O\left(m \cdot K^{\param} + K^{(1-\param)^{m}} \right) \, \cdot T^{1/B}\log(KTB)  \sum_{j:\epsilon_j > 0}  \frac{ 1}{\epsilon_j}.
$$ 
\end{theorem}
Thus, for any constant $\param \in (0, 1)$, 
setting $m = \frac{1}{\param}\log\left(\frac{1}{\param}\right)$, we obtain
expected regret bounded by
$$ \E[R(T)] = O\left( K^{\param} \,\, T^{1/B} \,\, \log(T)\right)\sum_{j:\epsilon_j > 0}\frac1{\epsilon_{j}}$$ 
in at most $B + \frac{1}{\param}\log\left(\frac{1}{\param}\right)$ rounds.
Conversely, given a value of $m$, we can appropriately select $\param$ to minimize the regret. Table~\ref{table:m-param} lists our instance-dependent regret bounds for some values of $m$. 

\begin{table}
    \centering
    {\renewcommand{\arraystretch}{2}%
    \begin{tabular}{|c|c|}
        \hline
        Rounds $B+m$  & Regret \\
        \hline
        $B+2$ &   $ K^{0.39} \,\, T^{1/B} \,\, \log(T) \,\, E$ \\
        \hline
        $B+3$ &  $K^{0.32} \,\,  T^{1/B} \,\, \log(T) \,\, E$  \\ 
        \hline
        $B+5$ &  $K^{0.25} \,\,  T^{1/B} \,\, \log(T) \,\, E$\\ 
        \hline
        $B+10$ &  $K^{0.17} \,\,  T^{1/B} \,\, \log(T) \,\, E$\\ 
        \hline
    \end{tabular}
    }
    \caption{Instance-dependent regret bounds vs. rounds in Theorem~\ref{thm:sst-sti-general-init}; here, $E=\sum_{j:\epsilon_j > 0}\frac1{\epsilon_{j}}$.}
    \label{table:m-param}
\end{table}
The idea behind this algorithm is to use a seed-set of size $K^\param$, and to recurse when the seed-set  is found to be sub-optimal. We bound the number of recursive calls by $m$ (which ensures that there are at most  $B+m$ rounds) and show that the active set shrinks by a shrinks by a power of $(1-\param)$ in each  recursive call (which is used to bound regret).

Finally, we complement our upper bound results with a lower bound for the batched $K$-armed dueling bandits problem, even under the  SST and STI 
assumptions. 
\begin{theorem}\label{thm:lb-init}
Given an integer $B > 1$, and 
any algorithm that uses at most $B$ batches,
there exists an instance of the $K$-armed batched dueling bandit problem
that satisfies 
the 
SST and STI condition
such that the expected regret $$ \E[R(T)] = \Omega\left( \frac{K T^{1/B}}{B^2\epsilon_{\min}} \right) .$$
\end{theorem}
The above lower bound shows that the $T^{1/B}$ dependence in our upper bounds is necessary.
Note that the above lower bound also applies to the more general Condorcet winner setting.
The proof is similar to the lower bound proof in~\cite{Gao+19} for batched multi-armed bandits.
The main novelty in our proof is the design of a family of hard instances with different values of $\epsilon_{\min}$'s
that satisfy the SST and STI conditions.

\section{Related Work}\label{sec:related-work}

The dueling bandits problem has been widely studied in recent years;
we mention the most relevant works here and refer the 
reader to \cite{SuiZHY18} for a more comprehensive survey.
This problem was first studied by 
\cite{YueBK+12} under the SST and STI setting.
The authors gave 
a worst-case regret upper bound of $\widetilde{O}(K\log T/\epsilon_{\min})$ and provided a matching lower bound.
\cite{YueJo11} considered a slightly more general version of the 
SST and STI setting and achieved an instance-wise optimal 
regret upper bound of $\sum_{j: \epsilon_j > 0} O\left(\frac{\log (T)}{\epsilon_{j}}\right)$.
\cite{Urvoy+13} studied this problem 
under the Condorcet winner setting and 
proved a $O(K^2\log T/\epsilon_{\min})$ regret upper bound, which was 
improved by \cite{Zoghi+14} 
to  $O(K^2/\epsilon_{\min}) + \sum_{j: \epsilon_j > 0}O( \log T/\epsilon_{j})$.
\cite{Komiyama+15a} achieved a similar but tighter KL divergence-based bound, which is shown to be \emph{asymptotically instance-wise optimal} (even in terms constant factors).
There are also other works that improve the dependence on $K$ in the upper bound, but 
suffer a worse dependence on $\epsilon_j$s \cite{Zoghi+15a}.
This problem has also been studied 
under other noise models such as utility based models \cite{Ailon+14}
and other notions of regret \cite{ChenFr17}.
Alternate notions of winners
such as Borda winner \cite{Jamieson+15}, Copeland winner \cite{Zoghi+15, Komiyama+16, WuLiu16}, and von Nuemann winner  \cite{Dudik+15} have also been considered.
There are also several works on extensions of dueling bandits that allow multiple arms to be compared at once \cite{Sui+17, Agarwal+20, SahaG19}.

All of the aforementioned works on the dueling bandits problem are limited to the sequential setting.
To the best of our knowledge, ours is the first work that considers the batched setting for dueling bandits.
However, batched processing for the stochastic multi-armed bandit problem has been investigated in the past few years.
A special case when there are two bandits was studied by~\cite{PerchetRC+16}. They obtain a worst-case regret bound of $O\left(\left(\frac{T}{\log(T)}\right)^{1/B}\frac{\log(T)}{\epsilon_{\min}}\right)$.
\cite{Gao+19} studied the general problem and obtained a worst-case regret bound of 
$O\left(\frac{K\log(K)T^{1/B}\log(T)}{\epsilon_{\min}}\right)$, which was later improved by \cite{EsfandiariKM+21} to $O\left(\frac{KT^{1/B} \log(T)}{\epsilon_{\min}}\right)$. Furthermore, \cite{EsfandiariKM+21} obtained an instance-dependent regret bound of $\sum_{j:\epsilon_j > 0}T^{1/B}O\left(\frac{\log(T)}{\epsilon_j}\right)$.
Our results for batched dueling bandits are of a similar flavor; that is, we get a similar dependence on
$T$ and $B$.
\cite{EsfandiariKM+21} also give batched algorithms for stochastic linear bandits and adversarial multi-armed bandits.

Adaptivity and batch processing has been recently studied for stochastic submodular cover~\cite{GolovinK-arxiv, AAK19, EsfandiariKM19, GGN21}, and for various stochastic ``maximization'' problems such as knapsack~\cite{DGV08,BGK11}, matching~\cite{BGLMNR12,BehnezhadDH20}, probing~\cite{GN13} and orienteering~\cite{GuhaM09,GuptaKNR15,BansalN15}.
Recently, there have also been several results examining the role of adaptivity in (deterministic) submodular optimization; e.g. ~\cite{BalkanskiS18,BalkanskiBS18,BalkanskiS18b,BalkanskiRS19,ChekuriQ19}.

\section{Algorithms for Batched Dueling Bandits}\label{sec:algs}

In this section, we present three algorithms, namely \pcomp, \scomp \ and \scomp2, 
for the $K$-armed batched dueling bandits problem. 
Recall that given a set of $K$ bandits (or arms) $\B = \{b_1, \ldots, b_K\}$, and a positive integer $B \leq T$, we wish to find a sequence of $B$ batches of noisy comparisons with low regret.
Given bandits $b_i$ and $b_j$, $P_{i, j} = \frac{1}{2} + \epsilon_{i, j}$ denotes the probability of $b_i$ winning over $b_j$. 
The first algorithm, termed \pcomp, proceeds by performing all-pairs comparisons 
amongst bandits in an \emph{active} set, and gradually eliminating sub-optimal bandits. The other two algorithms, termed \scomp \ and \scomp$2$, first select a (sufficiently small) \emph{seed} set $\cS \subset \B$, and eliminate bandits in an \emph{active} set by successively comparing them to (all or few) bandits in $\cS$.
If the seed set $\cS$ is itself found to be \emph{sub-optimal} in a subsequent round, then these algorithms call the  
all-pairs  algorithm \pcomp \ over the remaining \emph{active} arms.

Before describing our algorithms in detail we will set 
up some basic notation. 
We will denote by $\A$ the set of \emph{active} arms, i.e.\ arms that have not been eliminated. 
We will use index $r$ for rounds or batches.  
At the end of each round $r$, our algorithms compute a fresh estimate of the pairwise probabilities based on the feedback from comparisons in round $r$ as:
\begin{align}
\label{eq:pair_est}
\widehat{P}_{i, j} = \frac{\# b_i\text{ wins against } b_j \text{ in round } r}{\# \text{comparisons of } b_i \text{ and } b_j  \text{ in round } r }
    \,.
\end{align}
If a pair $(b_i, b_j)$ is compared in round $r$, it is compared $c_r = \lfloor q^r \rfloor$ times. In round $r$, the parameter $\gamma_r = \sqrt{{\log\left(\frac{1}{\delta}\right)}/{2c_r}}$ is used to eliminate bandits from the active set (the specific elimination criteria depends on the algorithm).

\ifConfVersion
\else
\begin{algorithm}
\caption{\pcomp (\textsc{All Pairs Comparisons Algorithm})}
\label{alg:all-comp}
\begin{algorithmic}[1]
\State \textbf{Input:} Bandits $\B$, time-horizon $T$, rounds $B$, comparison parameters $q$ and $\tau$
\State $K \gets |\B|$, $\delta \gets \frac{1}{6TK^2B}$, active bandits $\A \gets \B$, $c_r \gets  \lfloor q^{r+ \tau -1} \rfloor$, $\gamma_r \gets\sqrt{{\log(1/\delta)}/{2c_r}}$, $r \gets 1$
\While{number of comparisons $\leq T$}
\State for all $(b_i, b_j) \in \A^2$, perform $c_r$ comparisons and compute $\widehat{P}_{i,j}$ using Eq(\ref{eq:pair_est}). 
\If{$\exists$ $b_i, b_j$ such that $\widehat{P}_{i,j} > \frac{1}{2} + \gamma_r$}
\State $\A \gets \A \setminus \{b_j\}$ \Comment{delete $b_j$ from $\A$}
\EndIf
\State $r \gets r+1 $
\EndWhile
\end{algorithmic}
\end{algorithm}
\fi

\subsection{All Pairs Comparisons } \label{subsec:all-comp}
We first describe the \pcomp \ algorithm. This algorithm takes as input the set of bandits $\B$, time-horizon $T$, rounds $B$ and comparison parameters $q$ and $\tau$.
We will set the parameters $q = T^{1/B}$
and $\tau = 1$, unless otherwise specified.\footnote{We allow general parameters $q$ and $\tau$ in order to allow \pcomp \ to be used in conjunction with other policies.}
In  round $r \in [B]$, this algorithm compares each 
pair $(b_i, b_j) \in \A^2$ for $c_r$ times. 
It then computes fresh estimates of the pairwise probabilities $\widehat{P}_{i,j}$ for all $(b_i, b_j) \in \A^2$.
If, for some bandit $b_j$, there exists bandit $b_i$ 
such that $\widehat{P}_{i,j} > \frac{1}{2} + \gamma_r$, then bandit $b_j$ is eliminated from $\A$.
We provide the pseudo-code in Algorithm~\ref{alg:all-comp}.

\ifConfVersion
\else
The following theorem (proved in \S\ref{sec:proofs}) describes the regret bound obtained by \pcomp \ under the Condorcet assumption, and formalizes \Cref{thm:condorcet-init}.
\fi

\begin{theorem}\label{thm:condorcet}
Given any set $\B$ of $K$ bandits, time-horizon $T$, rounds $B$, parameters $q = T^{1/B}$ and $\tau = 1$,
the expected regret of \pcomp \ for the batched $K$-armed dueling bandits problem under the Condorcet 
assumption is at most
$$ \E[R(T)] \leq 3KT^{1/B} \log\left(6TK^2B\right) \sum_{j: \epsilon_{j} > 0} \frac{1}{\epsilon_{j}}.$$
Setting $\epsilon_{\min} := \min_{j: \epsilon_j > 0} \epsilon_j$, we get 
$$ \E[R(T)] \leq  \frac{ 3K^2T^{1/B} \log\left(6TK^2B\right)}{\epsilon_{\min}}. $$
\end{theorem}

\subsection{Seeded Comparisons Algorithms}
In this section, we present two algorithms for the batched dueling bandits problem, namely \scomp \ and \scomp$2$.
The algorithms work in \emph{two phases}:
\begin{itemize}
    \item In the first phase, the algorithms sample a \emph{seed set} $\cS$ by including  each bandit from $\B$ \emph{independently}  with probability $1/\sqrt{K}$. 
    This seed set 
    is used to
    eliminate bandits from the active set $\A$. 
    \item Under  certain {\em switching} criteria, the algorithms enter the second phase which involves running algorithm \pcomp\ on some of the remaining bandits. 
\end{itemize}
The algorithms differ in how the candidate set is used to eliminate active bandits in the first phase. 

In \scomp, \emph{all} pairwise comparisons between $\cS$ (seed set) and $\A$ (active bandits) are performed. Specifically, 
in round $r$, every active bandit is compared with every bandit  
in $\cS$ for $c_r$ times. If, for some bandit $b_j$, there exists bandit $b_i$ 
such that $\widehat{P}_{i,j} > \frac{1}{2} + 3\gamma_r$, then bandit $b_j$ is eliminated (from $\A$ as well as $\cS$); note that the elimination criteria here is stricter than in \pcomp. If, in some round $r$, there exists bandit $b_j$ such that $b_j$ eliminates \emph{all} bandits $b_i \in \cS$, then the algorithm constructs a set $\A^* = \{b_j \in \A \mid \widehat{P}_{j, i} > \frac{1}{2} + \gamma_r \text{ for all } b_i \in \cS \}$, and invokes \pcomp\ on bandits $\A^*$ with starting batch $r$. 
This marks the beginning of the second phase, which continues until time $T$. 
We provide the pseudocode in \Cref{alg:sst-sti-dep}.

\ifConfVersion
\else
\begin{algorithm}
\caption{\scomp (\textsc{Seeded Comparisons Algorithm})}
\label{alg:sst-sti-dep}
\begin{algorithmic}[1]
\State \textbf{Input:} Bandits $\B$, time-horizon $T$, rounds $B$
\State $q \gets T^{1/B}$, $\delta \gets \frac{1}{6TK^2B}$, active bandits $\A \gets \B$, $c_r \gets  \lfloor q^{r} \rfloor$, $\gamma_r \gets\sqrt{{\log(1/\delta)}/{2c_r}}$, $r \gets 1$
\State $\cS \gets $ add elements from $\B$ into $\cS$ w.p. $1/\sqrt{K}$
\While{number of comparisons $\leq T$} \Comment{phase I}
\State for all $(b_i, b_j) \in \cS \times \A$, compare $b_i$ and $b_j$ for $c_r$ times and compute $\widehat{P}_{i,j}$
\If{$\exists b_i \in \cS$, $b_j \in \A$, $\widehat{P}_{i, j} > \frac{1}{2} + 3\gamma_r$} \Comment{elimination}
\State $\A \gets \A \setminus \{b_j\}$, $\cS \gets \cS \setminus \{b_j\}$ 
\EndIf
\If{$\exists b_j$ such that $\widehat{P}_{j, i} > \frac{1}{2} + 3\gamma_r$ for all $b_i \in \cS$} \Comment{switching}
\State construct set $\A^* = \{ b_j \in \A \mid \widehat{P}_{j, i} > \frac{1}{2} + \gamma_r \text{ for all } b_i \in \cS \}$
\State $r^* \gets r$, $T^* \gets \text{\# comparisons until round } r^*$, \textbf{break}
\EndIf
\State $r \gets r+1$
\EndWhile
\State run \pcomp$(\A^*, T - T^*,  q, r^*)$ \Comment{phase II}
\end{algorithmic}
\end{algorithm}
\fi

\ifConfVersion
\else
We obtain the following result (proved in \S\ref{sec:proofs}) when the given instance satisfies SST and STI. This formalizes \Cref{thm:sst-sti-dep-init}.
\fi

\begin{theorem}\label{thm:sst-sti-dep}
Given any set $\B$ of $K$ bandits, time-horizon $T$, parameter $B$,
\scomp \  uses at most $B+1$ batches, and has expected regret bounded by $$ \E[R(T)] = \sum_{j: \epsilon_{j} > 0}O\left(\frac{\sqrt{K}T^{1/B} \log(T)}{\epsilon_{j}}\right) $$ under the strong stochastic transitivity and stochastic triangle inequality assumptions.
\end{theorem}

Observe that 
this
gives a worst-case regret bound of $O\left(\frac{K\sqrt{K}T^{1/B}\log(T)}{\epsilon_{\min}}\right)$ for \scomp \ under SST and STI. { We can improve this by sampling each bandit from $\B$ independently into the seed set with probability $K^{-2/3}$: this gives a worst-case regret bound of $O\left(\frac{K^{4/3}T^{1/B}\log(T)}{\epsilon_{\min}}\right)$ in $B+1$ rounds}.
To further improve this worst-case bound, we \emph{add more rounds of adaptivity} in \scomp \ to obtain \scomp2. Specifically, each round $r$ in the first phase is divided into two rounds of adaptivity. 
\begin{itemize}
    \item In the first round $r^{(1)}$, pairwise comparisons among the bandits in $\cS$ are performed, and an undefeated $b_{i^*_r}$ is selected as a \emph{candidate}. We say that $b_i$ \emph{defeats} $b_j$ if $\widehat{P}_{i, j} > \frac{1}{2} + \gamma_r$
    \item In the second round $r^{(2)}$, the candidate $b_{i_r^*}$ is used to eliminate active bandits. A bandit $b_j$ is eliminated if $\widehat{P}_{i_r^*, j} > \frac{1}{2} + 5\gamma_r$.
\end{itemize}
The switching criterion in \scomp2 is different from that of \scomp. Here, if in some round $r$, there is a  bandit $b_j$ such that $b_j$ eliminates $b_{i^*_r}$, then the algorithm constructs set $\A^* = \{b_j \in \A \mid \widehat{P}_{j, i_r^*} > \frac{1}{2} + 3\gamma_r\}$, and invokes \pcomp\ on bandits $\A^*$ with starting batch $r$. 
See \Cref{alg:sst-sti-indep} for a formal description.

\ifConfVersion
\else
\begin{algorithm}
\caption{\scomp$2$ (\textsc{Seeded Comparisons Algorithm $2$})}
\label{alg:sst-sti-indep}
\begin{algorithmic}[1]
\State \textbf{Input:} Bandits $\B$, time-horizon $T$, rounds $B$
\State $q \gets T^{1/B}$, $\delta \gets \frac{1}{6TK^2B}$,  active bandits $\A \gets \B$, $c_r \gets  \lfloor q^{r} \rfloor$, $\gamma_r \gets\sqrt{{\log(1/\delta)}/{2c_r}}$, $r \gets 1$
\State $\cS \gets $ add elements from $\B$ into $\cS$ with probability $1/\sqrt{K}$
\While{number of comparisons $\leq T$} \Comment{phase I}
\State  $r^{(1)}$: compare all pairs in $\cS$ for $c_r$ times and  compute $\widehat{P}_{i,j}$.
\State candidate $b_{i^*_r} \gets $ any  bandit  $i\in \cS$ with $\max_{j\in \cS}\widehat{P}_{j,i}\le \frac12+\gamma_r$.   \Comment{extra batch}
\State  $r^{(2)}$: for all $b_j \in \A$, compare $b_{i^*_r}$ and $b_j$ for $c_r$ times and  compute $\widehat{P}_{i^*_r,j}$. 
\If{$\exists b_j \in \A$, $\widehat{P}_{i_r^*, j} > \frac{1}{2} + 5\gamma_r$} \Comment{elimination}
\State $\A \gets \A \setminus \{b_j\}$, $\cS \gets \cS \setminus \{b_j\}$ 
\EndIf
\If{ $\exists b_j$ such that $\widehat{P}_{j, i^*_{r}} > \frac{1}{2} + 5\gamma_r$}\Comment{switching}
\State construct set $\A^* = \{ b_j \in \A \mid \widehat{P}_{j, i^*_{r}} > \frac{1}{2} + 3\gamma_r \}$
\State $r^* \gets r$, $T^* \gets \text{\# comparisons until round } r^*$, \textbf{break}
\EndIf
\State $r \gets r+1$
\EndWhile
\State run \pcomp$(\A^*, T-T^*, q, r^*)$ \Comment{phase II}
\end{algorithmic}
\end{algorithm}
\fi

\ifConfVersion
\else
In \S\ref{sec:proofs}, we prove that \scomp2 obtains an improved worst-case regret bound (at the cost of additional adaptivity) over \scomp \ when the given instance satisfies SST and STI, thus proving \Cref{thm:sst-sti-indep-init}.
\fi

\begin{theorem}\label{thm:sst-sti-indep}
Given any set $\B$ of $K$ bandits, time-horizon $T$ and parameter $B$,
\scomp2 \  uses at most $2B+1$ batches, and has worst-case expected regret bounded by $$ \E[R(T)] = O\left(\frac{KBT^{1/B} \log(T)}{\epsilon_{\min}}\right) $$ under   strong stochastic transitivity and stochastic triangle inequality, where
$\epsilon_{\min} := \min_{j: \epsilon_j > 0} \epsilon_j$. 
\end{theorem}

\ifConfVersion
\else
\section{Regret Analysis}\label{sec:proofs}
\fi
We present the regret analysis for the algorithms described in \S\ref{sec:algs} in this section.  
We first prove the following lemma which will be used in the analysis of all three algorithms.

\begin{lemma}\label{lem:confidence}
For any batch $r \in [B]$, and for any pair $b_i, b_j$ that are compared $c_r$ times, we have $$ \pr\left(|P_{i, j} - \widehat{P}_{i, j}| > \gamma_r \right) \leq 2\delta, $$ where $\gamma_r = \sqrt{\log(\frac{1}{\delta})/2c_r}$.
\end{lemma}
\begin{proof}
Note that $\E[\widehat{P}_{i, j}] = P_{i, j}$, and applying Hoeffding's inequality gives  
$$ \pr\left(|\widehat{P}_{i, j} - P_{i, j}| >\gamma_r\right) \leq 2\exp\left(-2c_r \cdot \gamma_r^2\right) = 2\delta. $$
\end{proof}

We analyze the regret of our algorithms under a \emph{good} event, $G$. 
We show that the $G$ occurs with high probability; in the event that $G$ does not occur (denoted $\overline{G}$), we incur a regret of $T$.
Towards defining $G$, we say that an estimate $\widehat{P}_{i,j}$ at the end of batch $r$ is \emph{correct} if $|\widehat{P}_{i, j} - P_{i, j}| \leq \gamma_r$. 
We say that $G$ occurs if every estimate in every batch is correct.
\begin{lemma}\label{lem:good-event}
The probability that every estimate in every batch of \pcomp, \scomp, and \scomp2 is correct is at least $1 - 1/T$.
\end{lemma}
\begin{proof}
Applying \Cref{lem:confidence} and taking a union bound over all pairs and batches (note \scomp2 has at most $2B+1 \leq 3B$ batches), we get 
that the probability that some estimate is incorrect is at most $K^2 \times 3B \times 2\delta = \frac{1}{T}$ where $\delta = 1/6K^2BT$. Thus, $\pr(\overline{G}) \leq \frac{1}{T}$. 
\end{proof}

Using \Cref{lem:good-event}, the expected regret (of \emph{any} algorithm) can be written as follows:
\begin{align}
    \E[R(T)] &= \E[R(T) \mid G] \cdot \pr(G) + \E[R(T) \mid \overline{G}] \cdot \pr(\overline{G}) \notag \\
    &\leq \E[R(T) \mid G] + T \cdot \frac{1}{T} = \E[R(T) \mid G] + 1 \label{eq:regret-1}
\end{align}

\begin{proof}[Proof of \Cref{thm:condorcet}]
First, recall that in each batch of \pcomp \, every pair of active arms is compared $c_r$ times where $c_r = \lfloor q^r \rfloor$ with $q = T^{1/B}$. Since, $q^B = T$,  \pcomp \ uses at most $B$ batches.

Following \Cref{lem:good-event} and \eqref{eq:regret-1}, we only need to bound $\E[R(T) \mid G]$. 
Given $G$, whenever $P_{i, j} > \frac{1}{2} + 2\gamma_r$ (that is $\epsilon_{i, j} > 2\gamma_r$), we have $\widehat{P}_{i, j} > \frac{1}{2} + \gamma_r$: so  bandit $b_j$ will be eliminated by $b_i$. Furthermore, given bandits $b_i$ and $b_j$ such that $b_i \succeq b_j$, $b_i$ will never be eliminated by $b_j$ under event $G$. 
This implies that $b_1$ is never eliminated: this is crucial as we use $b_1$ as an anchor to eliminate sub-optimal bandits.
Recall that 
the regret can be written as follows:
$$R(T) = \frac{1}{2}\sum_{j=1}^{K} T_j \epsilon_{1, j}$$
where $T_j$ is the number of comparisons that $b_j$ partakes in.
We proceed by bounding $T_j$. 
Towards this end, let $T_{1,j}$ be a random variable denoting the number of comparisons performed between $b_1$ and $b_j$.
As $b_1$ is never eliminated, $T_j \leq K \cdot T_{1, j}$. 
Let $r$ denote the last round such that $b_j$ survives round $r$, i.e., $b_j\in \A$ at the end of round $r$. We can then conclude that $\epsilon_j := \epsilon_{1,j} \leq 2\gamma_r$ (else $b_1$ would eliminate $b_j$ in round $r$). We get $$ \epsilon_{j} \leq2\cdot \sqrt{\frac{\log(\frac{1}{\delta})}{2c_r}} $$ which on squaring and re-arranging gives: 
\begin{equation}\label{eq:cr-bound}
   c_r  \leq \frac{2\log\left(\frac{1}{\delta}\right)}{\epsilon_{j}^2}  
\end{equation}
Now, note that $b_j$ could have been played for at most one more round. 
Thus, we have $$ T_{1, j} = \sum_{\tau=1}^{r+1} c_{\tau} \leq  q \sum_{\tau=0}^{r} c_{\tau} \leq 2q \cdot c_r$$ 
where the final inequality follows from summing up $\sum_{\tau=1}^{r-1}c_{\tau}$, and using $B \leq \log(T)$.
Then, we have $T_j \leq 2Kq \cdot c_r$. Using \ref{eq:cr-bound}, and plugging in $q = T^{1/B}$ and $\delta = 1/6TK^2B$ we have 
\begin{align*}
    E[R(T) \mid G] &\leq \frac{1}{2} \sum_{j}\left(2KT^{1/B} \cdot \frac{2\log\left(6TK^2B\right)}{\epsilon_{j}^2} \right) \cdot \epsilon_{j} \\
    &= \sum_{j: \epsilon_j > 0} \frac{KT^{1/B} \log\left(6TK^2B\right)}{\epsilon_{j}} \\
    &= 2KT^{1/B} \log\left(6TK^2B\right) \sum_{j: \epsilon_{j} > 0} \frac{1}{\epsilon_{j}}.\\
\end{align*}
{Note that when $\epsilon_j = 0$ for $b_j \in \B$, we exclude the corresponding term in the regret bound.} 
Combining this with \eqref{eq:regret-1} gives the first bound of \Cref{thm:condorcet}. Plugging in  $\epsilon_{\min} = \min_{j: \epsilon_{j} > 0} \epsilon_{j}$ completes the proof.
\end{proof}

\subsection{Proofs of Theorems~\ref{thm:sst-sti-dep} and \ref{thm:sst-sti-indep}}

In this section, we provide the proofs of \Cref{thm:sst-sti-dep} and \Cref{thm:sst-sti-indep}. Henceforth, we assume the SST and STI properties. We need the following definition.
For a bandit $b_j$, let $E_j = \{b_i \in \B: \epsilon_{i, j} > 0 \}$; that is, the set of bandits superior to bandit $b_j$. We define $rank(b_j) = |E_j|$. \footnote{{ Note that SST and STI imposes a linear ordering on the bandits. So, we can assume $b_1 \succeq b_2 \succeq \cdots \succeq b_K$. Thus, $rank(b_j) < j$; that is, it is at most the number of bandits strictly preferred over $b_j$.}} 

As before, 
we analyze the regret of \scomp \ and \scomp2 under event $G$. By Lemma~\ref{lem:good-event} and \eqref{eq:regret-1}, we only need to bound the expected regret under $G$; that is, we need to bound $\E[R(T) \mid G]$. Conditioned on event $G$, the following Lemmas~\ref{lem:a1},\ref{lem:a2} and \ref{lem:exp-rank} hold for both \scomp \ and \scomp2.

\begin{lemma}\label{lem:a1}
The best bandit $b_1$ is never deleted from $\A$ in the elimination step of phase I.
\end{lemma}
\begin{proof}
In \scomp, $b_i$ deletes $b_j$ in batch $r$ if $\widehat{P}_{i, j} > \frac{1}{2} + 3\gamma_r$, and in \scomp2 if $\widehat{P}_{i, j} > \frac{1}{2} + 5\gamma_r$. If $b_1$ is deleted due to some bandit $b_j$, then by applying \Cref{lem:confidence} (in either case), we get $P_{j, 1} > \frac{1}{2} + 2\gamma_r$, a contradiction.
\end{proof}

\begin{lemma}\label{lem:a2}
When the algorithm switches to \pcomp \ on set $\A^*$, we have $b_1 \in \A^*$ and $|\A^*| \leq rank(b_{i^*_{\cS}})$ where $b_{i^*_{\cS}}$ is the best bandit in $\cS$.
\end{lemma}
\begin{proof}
We first consider algorithm  \scomp. Here, the switching occurs when, in some batch $r$, there exists $b_{j^*} \in \A$ such that $\widehat{P}_{j^*, i} > \frac{1}{2} + 3\gamma_r$ for all $b_i \in \cS$, Moreover,  $\A^* = \{b_j \in \A \mid \widehat{P}_{j, i} > \frac{1}{2} + \gamma_r \text{ for all } b_i \in \cS\}$. Consider any $b_i \in \cS$. Given $G$, $\widehat{P}_{j^*, i} > \frac{1}{2} + 3\gamma_r$ implies that $P_{j^*, i} >\frac{1}{2} + 2\gamma_r$. By SST, $P_{1, i} \geq P_{j^*, i}$, and again using event $G$,   $\widehat{P}_{1, i} > \frac{1}{2} + \gamma_r$. Thus, $b_1 \in \A^*$. We now bound $|\A^*|$.  Let $b_{i^*_{\cS}}$ be the best bandit in $\cS$, i.e., the bandit of smallest rank. Consider any bandit $b_j \in \A^*$. We have  $\widehat{P}_{j, i^*_{\cS}} > \frac{1}{2} + \gamma_r$,  which implies (by event $G$) that $P_{j, i^*_{\cS}} > \frac{1}{2}$. So, we must have $b_j \succ b_{i^*_{\cS}}$.  Consequently, $\A^*\sse \{b_j \in \B : b_j \succ b_{i^*_{\cS}} \}$,   which implies $|\A^*| \leq rank(b_{i^*_{\cS}})$.

We now consider  \scomp2. Here, we select an \emph{undefeated} candidate bandit $b_{i^*_r}$ in batch $r$, and the algorithm switches   if there exists $b_{j^*} \in \A$ such that $\widehat{P}_{j^*, i^*_r} > \frac{1}{2} + 5\gamma_r$. Moreover,   $\A^* = \{b_j \in \A \mid \widehat{P}_{j, i_r^*} > \frac{1}{2} + 3\gamma_r \}$. Given $G$, we have $P_{j^*, i^*_r} > \frac{1}{2} + 4\gamma_r$. By SST and again applying $G$, we obtain $\widehat{P}_{1, i^*_r} > \frac{1}{2} + 3\gamma_r$. So, $b_1\in \A^*$. We now argue that $|\A^*| \le rank(b_{i^*_{\cS}})$. Again, let $b_{i^*_{\cS}}$ be the best bandit in $\cS$. As $b_{i^*_r}$ is undefeated after round $r^{(1)}$, we have $\widehat{P}_{i^*_{\cS}, i^*_r} \le \frac{1}{2} + \gamma_r$, which implies ${P}_{i^*_{\cS}, i^*_r} \le \frac{1}{2} + 2\gamma_r$  (by event $G$). Now, consider any bandit $b_j \in \A^*$. We have  $\widehat{P}_{j, i^*_{\cS}} > \frac{1}{2} + 3\gamma_r$,  which implies (by event $G$) that $P_{j, i^*_{\cS}} > \frac{1}{2}+ 2\gamma_r$. It follows that $b_j \succ   b_{i^*_{\cS}}$ for all  $b_j \in \A^*$. Hence,  $|\A^*| \leq rank(b_{i^*_{\cS}})$.\end{proof}

\begin{lemma}\label{lem:exp-rank}
We have $\E[rank(b_{i^*_{\cS}})] \le  \sqrt{K}$  and $\E[rank(b_{i^*_{\cS}})^2] \le  2K$.
\end{lemma}
\begin{proof}
The $R$ be a random variable denoting $rank(b_{i^*_{\cS}})$. Note that $R=k$ if, and only if, the first $k-1$ bandits are not sampled into $\cS$, and the $k^{th}$ bandit is sampled into $\cS$. Thus, $R$ is a geometric random variable with success probability $p:=\frac{1}{\sqrt{K}}$.\footnote{Strictly speaking, $R$ is truncated at $K$.} Recall that the mean and variance of a geometric random variable are $\frac1p$ and $\frac1{p^2}- \frac1p$ respectively. So, $\E[R] \le \frac{1}{p}= \sqrt{K}$. Moreover, $\E[R^2] \le   \frac{2}{p^2}=2K$.
\end{proof}

Using Lemmas~\ref{lem:a1}, \ref{lem:a2} and \ref{lem:exp-rank}, we complete the proofs of Theorems~\ref{thm:sst-sti-dep} and \ref{thm:sst-sti-indep}.

\begin{proof}[Proof of \Cref{thm:sst-sti-dep}]
We bound the expected regret of \scomp\ conditioned on $G$. Let $R_1$ and $R_2$ denote the regret incurred in phase I and II 
respectively. 

\paragraph{Bounding $R_1$.} Fix a bandit $b_j$. 
Let $r$ denote the last round such that $b_j\in \A$  {\em and} switching does not  occur (at the end of round $r$). Let $b_{i^*_{\cS}}$ be the best bandit in $\cS$. As $b_j$ is not eliminated by $b_{i^*_{\cS}}$, we have $\widehat{P}_{i^*_{\cS},j}\le \frac12 + 3\gamma_r$, which implies (by event $G$)  ${P}_{i^*_{\cS},j}\le \frac12 + 4\gamma_r$.  Moreover, as switching doesn't occur, we  have
 $\min_{i\in \cS} \widehat{P}_{1,i} \le \frac12 + 3\gamma_r$ (by  \Cref{lem:a1}, $b_1$ is never deleted from $\A$). 
 We now claim that ${P}_{1,i^*_{\cS}}\le \frac12 + 4\gamma_r$. Otherwise, by SST we have $\min_{i\in \cS} P_{1,i} = {P}_{1,i^*_{\cS}} > \frac12 + 4\gamma_r$, which (by event $G$) implies  $\min_{i\in \cS} \widehat{P}_{1,i} > \frac12 + 3\gamma_r$, a contradiction! It now follows that $\epsilon_{i^*_{\cS},j}\le 4\gamma_r$ and $\epsilon_{1,i^*_{\cS}}\le 4\gamma_r$. Consider now two cases: \begin{enumerate}
     \item  $b_1 \succeq b_{i^*_{\cS}} \succeq b_j$. Then, by STI, $\epsilon_{1, j} \leq 8\gamma_r$, and
     \item  $b_1 \succeq b_j \succeq b_{i^*_{\cS}}$. Then, by SST $\epsilon_{1, j} \leq \epsilon_{i^*_{\cS}, j} \leq 4\gamma_r$. 
 \end{enumerate}
 In either case, we have $\epsilon_j=\epsilon_{1, j} \leq 8\gamma_r$, which implies $c_r\le \frac{\log(1/\delta)}{2 \gamma_r^2}\le \frac{32 \log(1/\delta)}{\epsilon_j^2}$. 

Now, let 
$T_j$ be a random variable denoting  the number of comparisons of $b_j$ with other bandits before switching. By definition of round $r$, bandit $b_j$ will participate in at most one round after $r$ (in phase I). So, we have 
$$T_j \le\left\{ \begin{array}{ll}
    |\cS| \cdot \sum_{\tau=1}^{r+1}c_{\tau}  & \mbox{ if } b_j\not\in \cS\\
    K\cdot  \sum_{\tau=1}^{r+1}c_{\tau}  & \mbox{ if } b_j\in \cS\\
    \end{array}
\right.$$

Taking expectation over $\cS$, we get 
\begin{align}
\E\left[T_j\right] &\leq \E\left[K \sum_{\tau=1}^{r+1} c_{\tau} \,|\, b_j\in \cS\right] \cdot \pr(b_j \in \cS) +  \E\left[|\cS| \sum_{\tau=1}^{r+1}c_{\tau}\,|\, b_j \not\in \cS\right] \cdot \pr(b_j \notin \cS) \notag  \\
    &\leq \left(K \sum_{\tau=1}^{r+1} c_{\tau}\right) \cdot \frac{1}{\sqrt{K}} + \E[|\cS|\,|\, b_j \not\in \cS]\cdot \sum_{\tau=1}^{r+1}c_{\tau} \,\, \le \,\, 2\sqrt{K} \sum_{\tau=1}^{r+1}c_{\tau}
    \notag , \end{align}
where the third inequality uses $\E[|\cS|\,|\, b_j \not\in \cS] \le \sqrt{K}$. 
Moreover, 
$$\sum_{\tau=1}^{r+1}c_{\tau}  \le 2T^{1/B}\cdot c_r = O\left(\frac{T^{1/B} \log(1/\delta)}{\epsilon_j^2}\right).$$

Thus, 
\begin{equation} \label{eq:dep-r1}
    \E[R_1] = \sum_{j} \E\left[T_j\right] \cdot \epsilon_{ j} =\sum_{j: \epsilon_j > 0} O\left(\frac{T^{1/B}\sqrt{K}\log(6K^2TB)}{\epsilon_{j}}\right)
\end{equation}

\paragraph{Bounding $R_2$.}
We now bound the regret after switching. From Lemmas~\ref{lem:a1} and \ref{lem:a2}, we know that $b_1$ is never deleted, $b_1 \in \A^*$, and $|\A^*| \le rank(b_{i^*_{\cS}})$. For any $\A^*$, 
applying 
Theorem~\ref{thm:condorcet} we get, 
\begin{equation*} 
R_2  \leq 3|\A^*| T^{1/B}\log(6T|\A^*|^2B) \sum_{j\in \A^*:\epsilon_j > 0} \frac{1}{\epsilon_j}  \le 3|\A^*| T^{1/B}\log(6TK^2B) \sum_{j\in \B:\epsilon_j > 0} \frac{1}{\epsilon_j}  
\end{equation*}

By \Cref{lem:exp-rank}, $\E[|\A^*|] \le \sqrt{K}$, hence 
\begin{equation}\label{eq:dep-r2}
\E[R_2] \leq 3\sqrt{K} T^{1/B}\log(6TK^2B) \sum_{j:\epsilon_j > 0} \frac{1}{\epsilon_j}
\end{equation}

Combining \eqref{eq:dep-r1} and \eqref{eq:dep-r2}, we get $$ \E[R(T)|G] = \sum_{j: \epsilon_j > 0}O\left(\frac{T^{1/B}\sqrt{K}\log(6K^2TB)}{\epsilon_{j}}\right), $$ and by \eqref{eq:regret-1}, this concludes the proof.
\end{proof}

\begin{proof}[Proof of \Cref{thm:sst-sti-indep}]
We bound the expected regret  conditioned on $G$. Let $R_1$ and $R_2$ denote the regret incurred in phase I and II 
respectively. 

\paragraph{Bounding $R_1$.} Fix a bandit $b_j$. 
Let $r$ denote any round such that $b_j\in \A$  {\em and} switching does not  occur (at the end of round $r$). As in the proof of Theorem~\ref{thm:sst-sti-dep}, we first show that $c_r= O\left( \frac{  \log(1/\delta)}{\epsilon_j^2}\right)$. Recall that  
 $b_{i^*_r}$ is the candidate in round $r$. 
  As $b_j$ is not eliminated by $b_{i^*_r}$, we have $\widehat{P}_{i^*_{r},j}\le \frac12 + 5\gamma_r$, which implies (by event $G$)  ${P}_{i^*_{\cS},j}\le \frac12 + 6\gamma_r$.  Moreover, as switching doesn't occur, we  have
 $\widehat{P}_{1,i^*_r} \le \frac12 + 5\gamma_r$ (by  \Cref{lem:a1}, $b_1$ is never deleted from $\A$). By event $G$, we get ${P}_{1,i^*_r} \le \frac12 + 6\gamma_r$.
 It now follows that $\epsilon_{i^*_{r},j}\le 6\gamma_r$ and $\epsilon_{1,i^*_{r}}\le 6\gamma_r$. Consider now two cases: \begin{enumerate}
     \item  $b_1 \succeq b_{i^*_{r}} \succeq b_j$. Then, by STI, $\epsilon_{1, j} \leq 12\gamma_r$, and
     \item  $b_1 \succeq b_j \succeq b_{i^*_{\cS}}$. Then, by SST $\epsilon_{1, j} \leq \epsilon_{i^*_{r}, j} \leq 6\gamma_r$. 
 \end{enumerate}
 In either case, we have $\epsilon_j=\epsilon_{1, j} \leq 12\gamma_r$, which implies $c_r\le \frac{\log(1/\delta)}{2 \gamma_r^2}= O\left( \frac{  \log(1/\delta)}{\epsilon_j^2}\right)$.

We further divide $R_1$ into two kinds of regret: $R_{1}^{(c)}$ and $R_{1}^{(n)}$ where $R_{1}^{(c)}$ refers to the regret incurred by candidate arms and $R_{1}^{(n)}$ is the regret incurred by non-candidate arms. 

\paragraph{Bounding   $R_{1}^{(n)}$.}
For any bandit $b_j$, let $T_j$ be a random variable denoting the number of comparisons of $b_j$ (in phase I) when $b_j$ is not a candidate. Also, let $r$ be the last round such that $b_j\in \A$ and switching doesn't occur. So, $b_j$ will participate in at most one round after $r$, and 
$$T_j \le\left\{ \begin{array}{ll}
 \sum_{\tau=1}^{r+1}c_{\tau}  & \mbox{ if } b_j\not\in \cS\\
    |\cS| \cdot  \sum_{\tau=1}^{r+1}c_{\tau}  & \mbox{ if } b_j\in \cS
    \end{array}
\right.$$

Taking expectation over $\cS$, we get 
\begin{align}
\E\left[T_j\right] &\leq \E\left[|\cS| \sum_{\tau=1}^{r+1} c_{\tau} \,|\, b_j\in \cS\right] \cdot \pr(b_j \in \cS) +  \E\left[ \sum_{\tau=1}^{r+1}c_{\tau}\,|\, b_j \not\in \cS\right] \cdot \pr(b_j \notin \cS) \notag  \\
    &\leq \sum_{\tau=1}^{r+1} c_{\tau} \cdot \left( \frac{1}{\sqrt{K}} \cdot \E[|\cS|\,|\, b_j \in \cS] +  1\right)  \,\, \le \,\, (2+\frac{1}{\sqrt{K}})\cdot  \sum_{\tau=1}^{r+1}c_{\tau}
    \notag , \end{align}
where the third inequality uses $\E[|\cS|\,|\, b_j \in \cS] \le 1+\sqrt{K}$.

Moreover, using $c_r=O\left( \frac{  \log(1/\delta)}{\epsilon_j^2}\right)$, we have 
$\sum_{\tau=1}^{r+1}c_{\tau}  = O\left(\frac{T^{1/B} \log(1/\delta)}{\epsilon_j^2}\right)$. Thus,
\begin{align} 
    \E[R_{1}^{(n)}] &= \sum_{j} \E\left[T_j\right] \cdot \epsilon_{j} \leq \sum_{j: \epsilon_j > 0} O\left(\frac{T^{1/B}\log\left(\frac{1}{\delta}\right)}{\epsilon_{1, j}}\right) \leq O\left(\frac{T^{1/B}K\log\left(\frac{1}{\delta}\right)}{\epsilon_{\min}}\right)\label{eq:indep-r1-1}
\end{align}

\paragraph{Bounding   ${R}_{1}^{(c)}$.}
Observe that if $b_j$ is a candidate in round $r$, then the regret incurred by $b_j$ in round $r$ is at most $K c_r\cdot\epsilon_{1, j}$. Also,  $c_{r-1} \leq O\left(\frac{\log\left(\frac{1}{\delta}\right)}{\epsilon_{ j}^2}\right)$ because  $b_j\in \A$ and switching hasn't occurred at end of  round $r-1$. 
Thus, we have $c_r = T^{1/B} c_{r-1} \leq O\left(\frac{T^{1/B}\log\left(\frac{1}{\delta}\right)}{\epsilon_{ j}^2}\right).$ 
We can thus write $$ R_1^{(c)} = \sum_{r=1}^B \sum_{j}  K c_r\cdot\epsilon_{ j}\cdot  \mathbb{I}\left[i_r^* = j\right], $$ where $\mathbb{I}\left[i_r^* = j\right]$ is an indicator random variable denoting whether $b_j$ was the candidate bandit in round $r$. Observe that there is exactly one candidate bandit, $b_{i^*_r}$, in each round. So, 
\begin{align}
    R_1^{(c)} &= K \sum_{r=1}^B c_r \epsilon_{ i^*_r} \leq K \sum_{r=1}^B O\left(\frac{T^{1/B}\log\left(\frac{1}{\delta}\right)}{\epsilon_{i_r^*}^2}\right) \cdot \epsilon_{  i^*_r}  \notag \\
    &=K \sum_{r=1}^B O\left(\frac{T^{1/B}\log\left(\frac{1}{\delta}\right)}{\epsilon_{ i_r^*}}\right) \leq O\left(\frac{T^{1/B}KB\log\left(\frac{1}{\delta}\right)}{\epsilon_{\min}}\right)\label{eq:indep-r1c}
\end{align}
Combining \eqref{eq:indep-r1-1} and \eqref{eq:indep-r1c}, we get 
\begin{equation}\label{eq:indep-r1}
    \E[R_1] \leq O\left(\frac{T^{1/B}KB\log\left(\frac{1}{\delta}\right)}{\epsilon_{\min}}\right)
\end{equation}

\paragraph{Bounding   ${R}_2$.} Finally, we bound the regret in phase II where we only have bandits $\A^*$. From Lemmas~\ref{lem:a1} and \ref{lem:a2}, we know that $b_1 \in \A^*$, and $|\A^*| \le rank(b_{i^*_{\cS}})$. For any $\A^*$, 
applying 
Theorem~\ref{thm:condorcet} we get, 
\begin{equation*} 
R_2  \leq 3|\A^*| T^{1/B}\log(6T|\A^*|^2B) \sum_{j\in \A^*:\epsilon_j > 0} \frac{1}{\epsilon_j}  \le 3|\A^*|^2\cdot  T^{1/B}\log(6TK^2B) \cdot   \frac{1}{\epsilon_{\min}}  
\end{equation*}

By \Cref{lem:exp-rank}, $\E[|\A^*|^2] \le 2K$, and so: 
\begin{equation}\label{eq:indep-r2}
\E[R_2] \leq \frac{6T^{1/B}K\log(6TK^2B)}{\epsilon_{\min}}
\end{equation}
Finally, combining \eqref{eq:indep-r1} and \eqref{eq:indep-r2} completes the proof.
\end{proof}

\section{A Recursive Algorithm for Batched Dueling Bandits}

In this section, we describe a recursive algorithm, termed \rscomp\, for batched dueling bandits. 
The prior algorithms (\scomp \ and \scomp2) rely both on the seed set eliminating sub-optimal arms and on the fact that if the seed set is found to be sub-optimal, we can substantially \emph{shrink the active set} and \emph{switch} to the pairwise comparisons policy. 
we generalize \scomp \ to \rscomp \ by requiring an input $m \geq 1$ and $\param \in (0,1)$.
At a high level, \rscomp \ recurses $m$ times before switching to \pcomp. We maintain the property that each time it recurses, the active set shrinks by a factor of $K^{\param}$.
Note that when $m=1$ and $\param=1/2$, we recover \scomp.

The algorithm takes as input the set of bandits $\B$, time-horizon $T$, comparison parameters $q$ and $\tau$, integers $B$ and $m$, and an accuracy parameter $\param \in (0,1)$. Initially, we set $q = T^{1/B}$, $\tau = 1$ and $\delta = \frac{1}{2TK^2(B+m)}$. If $m=0$, the algorithm executes \pcomp; else, the algorithm works in two phases:
\begin{itemize}
    \item In the first phase, the algorithm samples a \emph{seed set} $\cS$ by including  each bandit from $\B$ \emph{independently}  with probability $1/K^{1-\param}$. 
    This seed set is used to
    eliminate bandits from the active set $\A$ (like in \scomp). 
    \item Under a certain switching criteria, the algorithm recurses on the active set $\A$ with $m = m-1$.
\end{itemize}

If $m \geq 1$, \rscomp\ performs all pairwise comparisons between $\cS$ (seed set) and $\A$ (active bandits). Specifically, 
in round $r$, every active bandit is compared with every bandit  in $\cS$ for $c_r$ times. If, for some bandit $b_j$, there exists bandit $b_i$ 
such that $\widehat{P}_{i,j} > \frac{1}{2} + 3\gamma_r$, then bandit $b_j$ is eliminated (from $\A$ as well as $\cS$).
If, in some round $r$, there exists bandit $b_j$ such that $b_j$ eliminates \emph{all} bandits $b_i \in \cS$, then the algorithm constructs a set $\A^* = \{b_j \in \A \mid \widehat{P}_{j, i} > \frac{1}{2} + \gamma_r \text{ for all } b_i \in \cS \}$, and recurses \rscomp\ on bandits $\A^*$ with parameter $\tau = r$, and $q$ and $\delta$ as set before. 
Additionally, the number of recursive calls is set to $m-1$.
Observe that the elimination and the switching criteria are the same as in \scomp.
Note that comparisons at round $r$ are repeated when a recursive call is invoked, and since there are at most $m$ recursive calls, \rscomp\ uses at most  $B+m$ adaptive rounds.
We describe the algorithm formally in \Cref{alg:scomp-recursive}.

\begin{algorithm}
\caption{\rscomp (\textsc{Recursive Seeded Comparisons Algorithm})}
\label{alg:scomp-recursive}
\begin{algorithmic}[1]
\State \textbf{Input:} Bandits $\B$, time-horizon $T$, comparison parameters $q$, $\tau$, and $\delta$,  $\#$recursive calls $m$, accuracy $\param$
\State active bandits $\A \gets \B$, $c_r \gets  \lfloor q^{r} \rfloor$, $\gamma_r \gets\sqrt{{\log(1/\delta)}/{2c_r}}$, $r \gets \tau$
\If{$m = 0$} \Comment{base case}
\State run \pcomp$(\B, T, q, \tau)$
\EndIf
\State $\cS \gets $ add elements from $\B$ into $\cS$ w.p. $1/K^{1-\param}$
\While{number of comparisons $\leq T$}
\State for all $(b_i, b_j) \in \cS \times \A$, compare $b_i$ and $b_j$ for $c_r$ times and compute $\widehat{P}_{i,j}$
\If{$\exists b_i \in \cS$, $b_j \in \A$, $\widehat{P}_{i, j} > \frac{1}{2} + 3\gamma_r$} \Comment{elimination}
\State $\A \gets \A \setminus \{b_j\}$, $\cS \gets \cS \setminus \{b_j\}$ 
\EndIf
\If{$\exists b_j$ such that $\widehat{P}_{j, i} > \frac{1}{2} + 3\gamma_r$ for all $b_i \in \cS$} \Comment{switching}
\State construct set $\A^* = \{ b_j \in \A \mid \widehat{P}_{j, i} > \frac{1}{2} + \gamma_r \text{ for all } b_i \in \cS \}$
\State $r^* \gets r$, $T^* \gets \text{\# comparisons until round } r^*$, \textbf{break}
\EndIf
\State $r \gets r+1$
\EndWhile
\State run \rscomp$(\A^*, T - T^*,  q, r^*, \delta, m-1, \param)$ \Comment{recursive call}
\end{algorithmic}
\end{algorithm}

The following theorem, which formalizes \Cref{thm:sst-sti-general-init}, is the main result of this section.
\begin{theorem}\label{thm:sst-sti-general}
Given any set $\B$ of $K$ bandits, time-horizon $T$, integers $B$ and $m$, parameters $q = T^{1/B}$, $r=1$ and $\delta = \frac{1}{2TK^2(B+m)}$, and an accuracy parameter $\param > 0$,
\rscomp \ uses at most $B + m$ batches, and has expected regret bounded by 
$$ \E[R(T)] = \sum_{j:\epsilon_j > 0} O\left(\left(m \cdot K^{\param} + K^{(1-\param)^{m}} \right) \,\, \cdot \frac{ T^{1/B}\log(6K^2TB)}{\epsilon_j}\right) $$
under strong stochastic transitivity and stochastic triangle inequality. 
\end{theorem}

\subsection{The Analysis}
We now provide the regret analysis of \rscomp, and prove \Cref{thm:sst-sti-general}. 
To keep the exposition in this section self-contained, we restate \Cref{lem:confidence} which will be used to define a good event for \rscomp.

\begin{lemma}\label{lem:confidence-gen}
For any $r \in [B]$, and for any pair $b_i, b_j$ that are compared $c_r$ times, we have $$ \pr\left(|P_{i, j} - \widehat{P}_{i, j}| > \gamma_r \right) \leq 2\delta, $$ where $\gamma_r = \sqrt{\log(\frac{1}{\delta})/2c_r}$.
\end{lemma}

As before, 
we analyze the regret of \rscomp \ under event a \emph{good} event, $G$. 
We show that $G$ occurs with high probability; in the event that $G$ does not occur (denoted $\overline{G}$), we incur a regret of $T$.
Towards defining $G$, we say that an estimate $\widehat{P}_{i,j}$ at the end of batch $r$ is \emph{correct} if $|\widehat{P}_{i, j} - P_{i, j}| \leq \gamma_r$. 
We say that $G$ occurs if every estimate in every batch is correct.
\begin{lemma}\label{lem:good-event-gen}
The probability that every estimate in the execution of \rscomp \ is correct is at least $1 - 1/T$.
\end{lemma}
\begin{proof}
Applying \Cref{lem:confidence-gen} and taking a union bound over all pairs and batches (note \rscomp \ has at most $B+m$ batches), we get 
that the probability that some estimate is incorrect is at most $K^2 \times (B+m) \times 2\delta = \frac{1}{T}$ where $\delta = \frac{1}{2K^2T(B+m)}$. Thus, $\pr(\overline{G}) \leq \frac{1}{T}$. 
\end{proof}

Using \Cref{lem:good-event-gen}, the expected regret \rscomp \ can be written as follows:
\begin{align}
    \E[R(T)] &= \E[R(T) \mid G] \cdot \pr(G) + \E[R(T) \mid \overline{G}] \cdot \pr(\overline{G}) \notag \\
    &\leq \E[R(T) \mid G] + T \cdot \frac{1}{T} = \E[R(T) \mid G] + 1 \label{eq:regret-1-gen}
\end{align}

By Lemma~\ref{lem:good-event-gen} and \eqref{eq:regret-1-gen}, we only need to bound the expected regret under $G$; that is, we need to bound $\E[R(T) \mid G]$. 
Henceforth, we assume the SST and STI properties. 
Recall that for a bandit $b_j$, we define $E_j = \{b_i \in \B: \epsilon_{i, j} > 0 \}$; that is, the set of bandits superior to bandit $b_j$, and $rank(b_j) = |E_j|$.
Conditioned on event $G$, the following Lemmas~\ref{lem:b1},\ref{lem:b2} and \ref{lem:b3} hold for \rscomp.

\begin{lemma}\label{lem:b1}
The best bandit $b_1$ is never deleted.
\end{lemma}
\begin{proof}
In \rscomp, $b_i$ deletes $b_j$ in batch $r$ if $\widehat{P}_{i, j} > \frac{1}{2} + 3\gamma_r$. If $b_1$ is deleted due to some bandit $b_j$, then by applying \Cref{lem:confidence-gen}, we get $P_{j, 1} > \frac{1}{2} + 2\gamma_r$, a contradiction.
\end{proof}

\begin{lemma}\label{lem:b2}
When the algorithm invokes a recursive call on $\A^*$, we have $b_1 \in \A^*$ and $|\A^*| \leq rank(b_{i^*_{\cS}})$ where $b_{i^*_{\cS}}$ is the best bandit in $\cS$.
\end{lemma}
\begin{proof}
Let $\A$ denote the set of active bandits in the some execution of \rscomp. Note that a recursive call is invoked when, in some batch $r$, there exists $b_{j^*} \in \A$ such that $\widehat{P}_{j^*, i} > \frac{1}{2} + 3\gamma_r$ for all $b_i \in \cS$, Moreover,  $\A^* = \{b_j \in \A \mid \widehat{P}_{j, i} > \frac{1}{2} + \gamma_r \text{ for all } b_i \in \cS\}$. Consider any $b_i \in \cS$. Given $G$, $\widehat{P}_{j^*, i} > \frac{1}{2} + 3\gamma_r$ implies that $P_{j^*, i} >\frac{1}{2} + 2\gamma_r$. By SST, $P_{1, i} \geq P_{j^*, i}$, and again using event $G$,   $\widehat{P}_{1, i} > \frac{1}{2} + \gamma_r$. Thus, $b_1 \in \A^*$. We now bound $|\A^*|$.  Let $b_{i^*_{\cS}}$ be the best bandit in $\cS$, i.e., the bandit of smallest rank. Consider any bandit $b_j \in \A^*$. We have  $\widehat{P}_{j, i^*_{\cS}} > \frac{1}{2} + \gamma_r$,  which implies (by event $G$) that $P_{j, i^*_{\cS}} > \frac{1}{2}$. So, we must have $b_j \succ b_{i^*_{\cS}}$.  Consequently, $\A^*\sse \{b_j \in \B : b_j \succ b_{i^*_{\cS}} \}$,   which implies $|\A^*| \leq rank(b_{i^*_{\cS}})$.
\end{proof}

\begin{lemma}\label{lem:b3}
We have $\E[rank(b_{i^*_{\cS}})] \le  K^{(1-\param)}$.
\end{lemma}
\begin{proof}
The $R$ be a random variable denoting $rank(b_{i^*_{\cS}})$. Note that $R=k$ if, and only if, the first $k-1$ bandits are not sampled into $\cS$, and the $k^{th}$ bandit is sampled into $\cS$. Thus, $R$ is a geometric random variable with success probability $p:=\frac{1}{K^{1-\param}}$. 
Thus, $\E[R] = \frac{1}{p}= K^{(1-\param)}$.
\end{proof}

Using Lemmas~\ref{lem:b1}, \ref{lem:b2} and \ref{lem:b3}, we complete the proof of Theorem~\ref{thm:sst-sti-general}.

\begin{proof}[Proof of Theorem~\ref{thm:sst-sti-general}]
The proof proceeds by induction on $m$. When $m=0$, \rscomp \ runs \pcomp \ and the result follows by \Cref{thm:condorcet} (proving the base case). 
Now, suppose that $m \geq 1$. We bound the expected regret of \rscomp\ conditioned on $G$. Let $R_1$ and $R_2$ denote the regret incurred before and after the first recursive call.

\paragraph{Bounding $R_1$.} Fix a bandit $b_j$. 
Let $r$ denote the last round such that $b_j\in \A$  {\em and} we do not recurse at the end of round $r$. Let $b_{i^*_{\cS}}$ be the best bandit in $\cS$. As $b_j$ is not eliminated by $b_{i^*_{\cS}}$, we have $\widehat{P}_{i^*_{\cS},j}\le \frac12 + 3\gamma_r$, which implies (by event $G$)  ${P}_{i^*_{\cS},j}\le \frac12 + 4\gamma_r$.  Moreover, as switching doesn't occur, we  have
 $\min_{i\in \cS} \widehat{P}_{1,i} \le \frac12 + 3\gamma_r$ (by  \Cref{lem:b1}, $b_1$ is never deleted from $\A$). 
 By SST, we conclude that ${P}_{1,i^*_{\cS}}\le \frac12 + 4\gamma_r$.
 It now follows that $\epsilon_{i^*_{\cS},j}\le 4\gamma_r$ and $\epsilon_{1,i^*_{\cS}}\le 4\gamma_r$. Consider now two cases: \begin{enumerate}
     \item  $b_1 \succeq b_{i^*_{\cS}} \succeq b_j$. Then, by STI, $\epsilon_{1, j} \leq 8\gamma_r$, and
     \item  $b_1 \succeq b_j \succeq b_{i^*_{\cS}}$. Then, by SST $\epsilon_{1, j} \leq \epsilon_{i^*_{\cS}, j} \leq 4\gamma_r$. 
 \end{enumerate}
 In either case, we have $\epsilon_j=\epsilon_{1, j} \leq 8\gamma_r$, which implies $c_r\le \frac{\log(1/\delta)}{2 \gamma_r^2}\le \frac{32 \log(1/\delta)}{\epsilon_j^2}$. 

Now, let 
$T_j$ be a random variable denoting  the number of comparisons of $b_j$ with other bandits before the recursive call. By definition of round $r$, bandit $b_j$ will participate in at most one round after $r$ (before recursing). So, we have 
$$T_j \le\left\{ \begin{array}{ll}
    |\cS| \cdot \sum_{\tau=1}^{r+1}c_{\tau}  & \mbox{ if } b_j\not\in \cS\\
    K\cdot  \sum_{\tau=1}^{r+1}c_{\tau}  & \mbox{ if } b_j\in \cS\\
    \end{array}
\right.$$

Taking expectation over $\cS$, we get 
\begin{align}
\E\left[T_j\right] &\leq \E\left[K \sum_{\tau=1}^{r+1} c_{\tau} \,|\, b_j\in \cS\right] \cdot \pr(b_j \in \cS) +  \E\left[|\cS| \sum_{\tau=1}^{r+1}c_{\tau}\,|\, b_j \not\in \cS\right] \cdot \pr(b_j \notin \cS) \notag  \\
    &\leq \left(K \sum_{\tau=1}^{r+1} c_{\tau}\right) \cdot \frac{1}{K^{1-\param}} + \E[|\cS|\,|\, b_j \not\in \cS]\cdot \sum_{\tau=1}^{r+1}c_{\tau} \,\, \le \,\, 2K^{\param} \sum_{\tau=1}^{r+1}c_{\tau}
    \notag , \end{align}
where the third inequality uses $\E[|\cS|\,|\, b_j \not\in \cS] \le K^{\param}$. 
Moreover, 
$$\sum_{\tau=1}^{r+1}c_{\tau}  \le 2T^{1/B}\cdot c_r = O\left(\frac{T^{1/B} \log(1/\delta)}{\epsilon_j^2}\right).$$

Thus, 
\begin{equation} \label{eq:gen-r1}
    \E[R_1] = \sum_{j} \E\left[T_j\right] \cdot \epsilon_{ j} =\sum_{j: \epsilon_j > 0} O\left(\frac{K^{\param}T^{1/B}\log(6K^2TB)}{\epsilon_{j}}\right)
\end{equation}

\paragraph{Bounding $R_2$.}
We now bound the regret after a recursive call is invoked. 
From Lemmas~\ref{lem:b1} and \ref{lem:b2}, we know that $b_1$ is never deleted, $b_1 \in \A^*$, and $|\A^*| \le rank(b_{i^*_{\cS}})$. For any $\A^*$, 
on applying the inductive hypothesis with $m-1$ recursive calls, we get
\begin{align*} 
\E[R_2 \mid \A^*] &\leq \sum_{j:\epsilon_j > 0} O\left(\left((m-1) \cdot |\A^*|^{\param}+ |\A^*|^{(1-\param)^{m-1}}\right)\cdot \frac{T^{1/B}\log(6|\A^*|^2TB)}{\epsilon_j}\right) \\
&\leq \sum_{j:\epsilon_j > 0} O\left(\left((m-1) \cdot K^{\param}+ |\A^*|^{(1-\param)^{m-1}}\right)\cdot \frac{T^{1/B}\log(6K^2TB)}{\epsilon_j}\right) 
\end{align*}
since $|\A^*| \leq K$. Taking an expectation over $\A^*$, we obtain
\begin{equation*} 
\E[R_2] \leq \sum_{j:\epsilon_j > 0} O\left(\left((m-1) \cdot K^{\param}+ \E\left[|\A^*|^{(1-\param)^{m-1}}\right]\right)\cdot \frac{T^{1/B}\log(6K^2TB)}{\epsilon_j}\right) 
\end{equation*}
Finally, observe that by Jensen's inequality, we have $\E\left[|\A^*|^{(1-\param)^{m-1}}\right] \leq \E[|\A^*|]^{(1-\param)^{m-1}}$, and by Lemma~\ref{lem:b3} $
\E[|\A^*|] \leq K^{1-\param}$. 
Combining these observations, we get $\E\left[|\A^*|^{(1-\param)^{m-1}}\right] \leq K^{(1-\param)^m}$. Thus, we obtain
\begin{equation} \label{eq:gen-r2}
\E[R_2] \leq \sum_{j:\epsilon_j > 0} O\left(\left((m-1) \cdot K^{\param}+ K^{(1-\param)^{m}}\right)\cdot \frac{T^{1/B}\log(6K^2TB)}{\epsilon_j}\right) 
\end{equation}
Finally, combining \eqref{eq:gen-r1} and \eqref{eq:gen-r2}
completes the induction.
\end{proof}

\section{Lower Bound}\label{sec:lb}

\newcommand{\TV}{D_{\text{TV}}}

In this section, we present a lower bound for the batched dueling bandits problem under the SST and STI setting. Note that this lower bound also applies to the more general Condorcet winner setting. The main result of this section is the following:

\begin{theorem}\label{thm:lb}
Given an integer $B > 1$, and 
any algorithm that uses at most $B$ batches,
there exists an instance of the $K$-armed batched dueling bandit problem
that satisfies the SST and STI conditions
such that the expected regret $$ \E[R_T] = \Omega\left( \frac{K T^{1/B}}{B^2\epsilon_{\min}} \right) \,,$$
where $\epsilon_{\min}$ is defined with respect to the particular  instance. 
\end{theorem}

In order to prove this theorem,
we will construct a family of instances 
such that any algorithm for batched dueling bandits cannot simultaneously beat the above regret lower bound over all instances in the family. 
We exploit the fact that the algorithm is unaware of the particular instance chosen from the family at run-time, and hence, is unaware of the gap $\epsilon_{\min}$ under that instance.

\textbox{Family of Instances $\I$:}{
    \begin{itemize}
        \item Let $F$ be an instance where $P_{i,j} = \half$ for all $i,j \in \B$.
        \item For $j \in [B]$, let $\Delta_j = \frac{\sqrt{K}}{24 B} \cdot T^{(j-1)/2B}$. For $j \in [B]$ and $k \in [K]$, let $E_{j, k}$ be an instance where bandit $b_k$ is the Condorcet winner 
        such that $P_{k,l} = \half + \Delta_j$
        for all $l \in [K] \setminus\{k\}$
        and $P_{l,m} = \half$ for all $l,m \in [K]\setminus\{k\}$.
        \item The family of instances  $\I := \{E_{j,k}\}_{j \in [B], k \in [K]} \cup \{F\}$.
    \end{itemize}
}

\subsection{Proof of \Cref{thm:lb}}
Let us fix an algorithm $\A$ for this problem.
Let $T_j = T^{j/B}$ for $j \in [B]$. 
Let $t_j$  be the total (random) number of comparisons until the end of batch $j$ during the  execution of $\A$. 
We will overload notation and denote by $I^t$ 
the distribution of observations seen by the algorithm when the underlying instance is $I$. 
We will sometimes use $P_{i,j}(I)$ for the probability of $i$ beating $j$ under an instance $I$ to emphasize the dependence on $I$.
We will also write $\epsilon_{\min}(I)$ to emphasize 
the dependence on the underlying instance $I$.

We define event $A_j$ as follows: $$ A_j = \{ t_{j'} < T_{j'}, \forall j' < j \text{ and } t_j \geq T_j \}, $$ and denote by $E_{j, k}(A_j)$ the event that $A_j$ occurs given that the instance selected is $E_{j, k}$. Similarly, $F(A_j)$ denotes the event that $A_j$ occurs when the instance selected is $F$. Now, define $$ p_j = \frac{1}{K} \sum_{l=1}^K \pr(E_{j, l}(A_j)). $$ Observe that $p_j$ is the average probability of event $A_j$ conditional on the instance having gap $\Delta_j$.

\begin{lemma}\label{lem:lb-1}
$\sum_{j=1}^B p_j \geq \frac{1}{2}$.
\end{lemma}
\begin{proof}
Note that the event $A_j$ is determined by observations until $T_{j-1}$. This is because $t_{j-1} < T_{j-1}$, and once the observations until $t_{j-1}$ are seen: the next batch $j$ determines whether or not $A_j$ occurs. 
Hence, in order to bound the probability of $A_j$ under two different instances $F$ and $E_{j,l}$ we use 
the Pinsker's inequality as $$ |\pr(F(A_j)) - \pr(E_{j, l}(A_j))| \leq \sqrt{\frac{1}{2} D_{\text{KL}}(F^{T_{j-1}} || E^{T_{j-1}}_{j, l})} $$ for $l \in [K]$. Let $\tau_l$ be the random variable for the number of times arm $l$ is played until $T_{j-1}$.
We first bound 
$D_{\text{KL}}(F^{T_{j-1}} || E^{T_{j-1}}_{j, l})$
as 
\begin{align}
D_{\text{KL}}(F^{T_{j-1}} || E^{T_{j-1}}_{j, l})
    & \overset{(a)}{=}
    \sum_{t=1}^{T_{j-1}} D_{\text{KL}}\big(P_{t_1, t_2}(F) ~||~ P_{t_1, t_2}(E_{j,l})\big) \nonumber\\
    & \overset{(b)}{\leq}
    \sum_{t = 1}^{T_{j-1}} \Pr_F(\text{arm } l \text{ is played in trial }t) \cdot D_{\text{KL}}\left(\half ~||~ \half + \Delta_j\right) \nonumber \\
    & \overset{(c)}{\leq}
    \E_F[\tau_l] \cdot 4 \Delta_j^2 \label{eq:kl_bound}
        \,,
\end{align}
where $(a)$ follows from the fact that, given $F$, the outcome of comparisons are independent across trials,  $(b)$ follows from the fact that the KL-divergence between $P_{t_1, t_2}(F)$ and $P_{t_1, t_2}(E_{j,k})$ is non-zero only when arm $l$ is played in trial $t$,
and $(c)$ follows from the fact that 
$D_\text{KL}(p || q) \leq \frac{\paren{p-q}^2}{q \cdot (1-q)} $.
Using the above bounds, we have that
\begin{align*}
    \frac{1}{K} \sum_{l = 1}^K |\pr(F(A_j)) - \pr(E_{j, l}(A_j))| &\leq  \frac{1}{K} \sum_{l = 1}^K \sqrt{\frac{1}{2} D_{\text{KL}}(F^{T_{j-1}} || E^{T_{j-1}}_{j, l})} \\
    &\leq\frac{1}{K}\sum_{l=1}^K \sqrt{\frac{1}{2} \cdot 4\Delta_j^2 \E_F[\tau_{l}] } = \frac{1}{K}\sum_{l=1}^K \sqrt{2\Delta_j^2 \E_F[\tau_{l}] }\\
    & \overset{(a)}{\leq} \sqrt{\frac{2\Delta_j^2\E_F[\sum_{l=1}^K\tau_{l}]}{K}} \\
    &\overset{(b)}{\leq} \sqrt{\frac{2\Delta_j^2 \cdot 2 T_{j-1} }{K} } = \frac{1}{2B}
        \,,
\end{align*}
where $(a)$ follows from the concavity of $\sqrt{x}$ and Jensen's inequality, and $(b)$ follows from the fact that  $\sum_{l=1}^K\tau_{l} \leq T_{j-1}$.
We thus have
\begin{align*}
    |\pr(F(A_j)) - p_j| &= |\pr(F(A_j)) - \frac{1}{K} \sum_{l=1}^K \pr(E_{j, l}(A_j))| \\
        &\leq \frac{1}{K} \sum_{l=1}^K |\pr(F(A_j)) - \pr(E_{j, l}(A_j))| \leq \frac{1}{2B}
        \,.
\end{align*}

Finally, we can write 
\begin{align*}
    \sum_{j=1}^B p_j \geq \sum_{j=1}^B (\pr(F(A_j)) - \frac{1}{2B}) 
        \geq \sum_{j=1}^B\pr(F(A_j)) - \frac{1}{2} 
        \geq \frac{1}{2}
            \,.
\end{align*}
\end{proof}

As a consequence of this lemma, we can conclude that there exists some $j \in [B]$ such that $p_j \geq \frac{1}{2B}$. We focus on the event where gap is $\Delta_j$,
and prove that when $p_j \geq \frac{1}{2B}$, $\A$ must suffer a high regret leading to a contradiction. The next lemma formalizes this.

\begin{lemma}
If, for some $j$, $p_j \geq \frac{1}{2B}$, then $$\sup_{I : \epsilon_{\min}(I) = \Delta_j}\E[R_T(I)] \geq \Omega\left( \frac{K T^{1/B}}{B^2 \Delta_j} \right) $$
\end{lemma}
\begin{proof}
Fix $k \in [K]$. 
We will construct a family of instances
$\{Q_{j,k,l}\}_{l \neq k}$ 
where $Q_{j,k,l}$ is defined as:
\textbox{
Instance $Q_{j,k,l}$:}{
    Arm $l$ is the Condorcet winner and the pairwise preferences are defined as:
    \[
    P_{lm} = \half + 2 \Delta_j, \forall m \in [K] \setminus \{l\}; \qquad
    P_{km} = \half + \Delta_j, \forall m \in [K]\setminus \{l,k\} ;
    \]
    and $P_{m m'} = \half$ for remaining pairs $(m,m')$.
}
We also let $Q_{j,k,k} := E_{j,k}$.
Note that the regret is $\geq \Delta_j$
if the underlying instance is $Q_{j,k,l}$
and the pair played is not $(b_l,b_l)$.
We have that 
\[
\sup_{I : \epsilon_{\min}(I)= \Delta_j} \E[R_T(I)] \geq \Delta_j \sum_{t=1}^T \frac{1}{K} \sum_{l \neq k}Q_{j,k,l}^t \paren{ (b_{t_1}, b_{t_2}) \neq (b_l,b_l)}
    \,,
\]
where $Q_{j,k,l}^t$ denotes the distribution of observations available at time $t$ under instance $Q_{j,k,l}$ and $Q_{j,k,l}^t \paren{ (b_{t_1}, b_{t_2}) \neq (b_l,b_l)}$ is the probability that 
the algorithm does not play arm $(b_l, b_l)$ at time $t$ under $Q_{j,k,l}^t$.
In order to bound the above quantity we will need the 
following lemma from \cite{Gao+19}.
\begin{lemma}[Lemma 3 of \cite{Gao+19}]
Let $Q_1, \cdots Q_K$ be probability measures on some common probability space $(\Omega, \mathcal{F} )$, and $\psi : \Omega \rightarrow [K]$ be any measurable function (i.e., test). Then, for any tree $\mathcal{T} = ([K], E)$ with vertex set $[K]$ and edge set $E$, 
\[
\frac{1}{K} \sum_{i = 1}^K Q_i(\psi \neq i) \geq \frac{1}{K} \sum_{(l,l') \in E} \int \min\{dQ_l, dQ_{l'}\}
    \,.
\]
\end{lemma}
Using the above lemma for the star graph centered at $k$, we have that
\begin{align*}
   \sup_{I : \epsilon_{\min}(I)= \Delta_j} \E[R_T(I)] & \geq \Delta_j \sum_{t=1}^T \frac{1}{K}\sum_{l \neq k} \int \min\{dQ_{j,k,k}^t , dQ_{j,k,l}^t\} \\
        & \overset{(a)}{\geq} \Delta_j \sum_{t=1}^{T_j} \frac{1}{K}\sum_{l \neq k} \int \min\{dQ_{j,k,k}^t , dQ_{j,k,l}^t\} \\
        & \overset{(b)}{\geq}  \Delta_j \sum_{t=1}^{T_j} \frac{1}{K}\sum_{l \neq k} \int \min\{dQ_{j,k,k}^{T_j} , dQ_{j,k,l}^{T_j}\} \\
        & {\geq}   \Delta_j \sum_{t=1}^{T_j} \frac{1}{K}\sum_{l \neq k} \int_{A_j} \min\{dQ_{j,k,k}^{T_j} , dQ_{j,k,l}^{T_j}\} \\
        & \overset{(c)}{\geq}  \Delta_j \sum_{t=1}^{T_j} \frac{1}{K}\sum_{l \neq k} \int_{A_j} \min\{dQ_{j,k,k}^{T_{j-1}} , dQ_{j,k,l}^{T_{j-1}}\} \numberthis \label{eqn:int1}
        \,,
\end{align*}
where $(a)$ follows because $T_j \leq T$, $(b)$
follows due to the fact that $\int \min\{dP,dQ\} = 1-\TV(P,Q)$ and the fact that $\TV(Q_{j,k,k}^{T_j},Q_{j,k,l}^{T_j})$
is at least $\TV(Q_{j,k,k}^{t},Q_{j,k,l}^{t})$
as the sigma algebra $\mathcal{F}_{Q_{j,k,k}^{t}}$ of 
$Q_{j,k,k}^{t}$ is a subset of the sigma algebra  
$\mathcal{F}_{Q_{j,k,k}^{T_j}}$ of $Q_{j,k,k}^{T_j}$, $(c)$ follow
from the fact that the event $A_j$ is determined by observations until $T_{j-1}$ as explained in the proof of Lemma~\ref{lem:lb-1}.
We then have that 
\begin{align*}
    \int_{A_j} \min\{dQ_{j,k,k}^{T_{j-1}} , dQ_{j,k,l}^{T_{j-1}}\} &= 
    \int_{A_j} \frac{dQ_{j,k,k}^{T_{j-1}}+  dQ_{j,k,l}^{T_{j-1}} - |dQ_{j,k,k}^{T_{j-1}}-  dQ_{j,k,l}^{T_{j-1}}|}{2} \\
    & = \frac{Q_{j,k,k}^{T_{j-1}}(A_j) + Q_{j,k,l}^{T_{j-1}}(A_j)}{2} - \int_{A_j} \frac{ |dQ_{j,k,k}^{T_{j-1}}-  dQ_{j,k,l}^{T_{j-1}}|}{2} \\
    & \overset{(a)}{\geq} Q_{j,k,k}^{T_{j-1}}(A_j) - \half \TV\paren{Q_{j,k,k}^{T_{j-1}}, Q_{j,k,l}^{T_{j-1}}} - \TV\paren{Q_{j,k,k}^{T_{j-1}}, Q_{j,k,l}^{T_{j-1}}} \\
    & =Q_{j,k,k}^{T_{j-1}}(A_j) - \frac{3}{2} \TV\paren{Q_{j,k,k}^{T_{j-1}}, Q_{j,k,l}^{T_{j-1}}} \numberthis \label{eqn:int2}
        \,,
\end{align*}
where $(a)$ follows from the fact that $\TV(P,Q) = \sup_{A} |P(A) - Q(A)|$.
Let us define $\tau_l$ to be the random variable for the number of times arm $l$ is played until $T_{j-1}$
We also have that 
\begin{align*}
    \frac{1}{K} \sum_{l \neq k } \TV\paren{Q_{j,k,k}^{T_{j-1}}, Q_{j,k,l}^{T_{j-1}}}  &\leq  \frac{1}{K} \sum_{l \neq k } \sqrt{\frac{1}{2} D_{\text{KL}}(Q_{j,k,k}^{T_{j-1}} || Q_{j,k,l}^{T_{j-1}})} \\
    & \overset{(a)}{\leq} \frac{1}{K} \sum_{l \neq k } \sqrt{\frac{1}{2} \cdot 16\Delta_j^2 \E_{E_{j,k}}[\tau_{l}]} = \frac{1}{K}\sum_{l\neq k} \sqrt{8\Delta_j^2 \E_{E_{j,k}}[\tau_{l}] }\\
    & \overset{(b)}{\leq} \sqrt{\frac{8\Delta_j^2\E_{E_{j,k}}[ \sum_{l \neq k } \tau_{l}]}{K}} \\
    &\overset{(c)}{\leq} \sqrt{\frac{8\Delta_j^2}{K} 2T_{j-1}} =\frac{1}{6B} \numberthis \label{eqn:int3}
        \,,
\end{align*}
where $(a)$ follows from a similar calculation as \cref{eq:kl_bound} in the proof of Lemma~\ref{lem:lb-1}, 
$(b)$ follows from the concavity of $\sqrt{x}$ and Jensen's inequality, and $(c)$ follows from the fact that  $\sum_{l=1}^K\tau_{l} \leq T_{j-1}$.

Combining \cref{eqn:int1,eqn:int2,eqn:int3} we have that 
\begin{align*}
   \sup_{I : \epsilon_{\min}(I)= \Delta_j} \E[R_T(I)] & \geq \Delta_j T_j \paren{\pr(E_{j,k}(A_j)) - \frac{1}{4B}}
        \,.
\end{align*}
Since the above inequality holds for all $k \in [K]$,
by averaging we get
\begin{align*}
   \sup_{I : \epsilon_{\min}(I)= \Delta_j} \E[R_T(I)] & \geq \Delta_j T_j \paren{ \frac{1}{K}\sum_{k=1}^K\pr(E_{j,k}(A_j)) - \frac{1}{4B}} \\
    & = \Delta_j T_j \paren{p_j - \frac{1}{4B}} \\
    & \geq \Delta_j T_j \frac{1}{4B}
        \,.
\end{align*}
Substituting the value of $\Delta_j T_j$ we get
\begin{align*}
    \sup_{I : \epsilon_{\min}(I)= \Delta_j} \E[R_T(I)] & \geq \Delta_j T_j \frac{1}{4B}
        = \frac{\sqrt{K}}{24 B} T^{-(j-1)/2B} T^{j/B} \frac{1}{4B} \\
        & = \frac{\sqrt{K}}{24 B} T^{(j-1)/2B} T^{1/B} \frac{1}{4B}
        =  \Omega\paren{\frac{K T^{1/B}}{B^2 \Delta_j}}
            \,.
\end{align*}

\end{proof}

Finally,
$\sum_{j=1}^B p_j \geq \half$ implies that there exists
$j \in [B]$ with $p_j \geq 1/2B$.
Combining the two lemmas above, we get that 
there exists $j \in [B]$ with $p_j \geq 1/2B$ such that the algorithm 
incurs a regret of $\Omega\paren{\frac{K T^{1/B}}{B^2 \Delta_j}}$.
In this case, there must exist an instance $E_{j,k}$
with gap $\epsilon_{\min}(E_{j,k}) = \Delta_j$
such that the regret of the algorithm 
under $E_{j,k}$ is $\Omega\paren{\frac{K T^{1/B}}{B^2 \Delta_j}}$.
This completes the proof of our lower bound.

\section{Experimental Results} \label{sec:comp-results}

We provide a summary of computational results of our algorithms for the batched  dueling bandits problem. We conducted our computations using C++ and Python 2.7 with a $2.3$ Ghz Intel Core $i5$ processor and $16$ GB $2133$ MHz LPDDR3 memory.

\textbf{Experimental Setup.}  
We compare all our algorithms, namely \pcomp, \scomp, and \scomp2 to a representative set of sequential algorithms for dueling bandits. Specifically, we use the 
dueling bandit library due to \cite{Komiyama+15a}, 
and compare our algorithms to RUCB~\cite{Zoghi+14}, RMED1~\cite{Komiyama+15a}, and \textsc{Beat-the-Mean}~\cite{YueJo11}. Henceforth, we refer to \textsc{Beat-the-Mean} as BTM.
We plot the cumulative regret $R(t)$ incurred by the algorithms against time $t$.
Furthermore, to illustrate the dependence on $B$, we run another set of experiments on \scomp2 and plot the cumulative regret $R(t)$ incurred by \scomp2 against time $t$ for varying values of $B$.\footnote{We also conducted these experiment for \pcomp \ and \scomp \ and the conclusions were similar.}
We perform these experiments 
using both real-world
and synthetic data.
We use the following datasets:

\textbf{Six rankers.} This real-world dataset is based on the $6$ retrieval functions used in the engine of ArXiv.org.

\textbf{Sushi.} The Sushi dataset is based on the Sushi preference dataset~\cite{Kamishima03} that contains the preference data regarding $100$ types of Sushi. A preference dataset using the top-$16$ most popular types of sushi is obtained.

\textbf{BTL-Uniform.} We generate synthetic data using the Bradley-Terry-Luce (BTL) model. Under this model, each arm $b_i\in \B$ is associated 
with a weight $w_i > 0$ (sampled uniformly in the interval $(0,1]$), and we set $P_{i,j} = w_{i}/(w_i+w_j)$. We set the number of arms $K = 100$. Note that the data generated in this way satisfies SST and STI~\cite{YueBK+12}. We refer to this data as $\mathtt{SYN}$-$\mathtt{BTL}$.

\textbf{Hard-Instance.} The last dataset is a synthetic dataset inspired by the hard instances that we construct for proving our lower bound~(see \Cref{thm:lb-init}). Again, we set $K=100$, and pick $\ell \in [K]$ uniformly at random as the Condorcet winner. We select $\Delta$ uniformly in $(0, 0.5)$, and set $P_{\ell, i} = \frac{1}{2} + \Delta$ for $i \neq \ell$. Furthermore, for all $i, j \neq \ell$, we set $P_{i, j} = 1/2$. We refer to this data as $\mathtt{SYN}$-$\mathtt{CD}$.

Note that there exists a Condorcet winner in all datasets. Moreover, the $\mathtt{SYN}$-$\mathtt{BTL}$ dataset satisfies SST and STI. We repeat each experiment $10$ times and report the average regret. 
In our algorithms, we {use the KL-divergence based confidence bound (as in RMED1) for elimination as it performs much better empirically (and our  theoretical bounds continue to hold). In particular, we replace lines $5$, $6$ and $8$ in \pcomp, \scomp\ and \scomp2, respectively, with KL-divergence based 
elimination criterion that eliminates an arm $i$ if there exists another arm $j$ if $\hat{P}_{ij} < \half$ and $N_{ij} \cdot D_{\text{KL}}(\hat{P}_{ij}, \half) > \log(T\delta)$ where $N_{ij}$ is the number of times arm $i$ and $j$ are played together.
} We report the average cumulative regret at each time step.

\textbf{Comparison with sequential dueling bandit algorithms.}
\ifConfVersion
\else
\begin{figure}[t]
     \centering
     \begin{subfigure}[b]{0.45\textwidth}
         \centering
         \includegraphics[width=\textwidth]{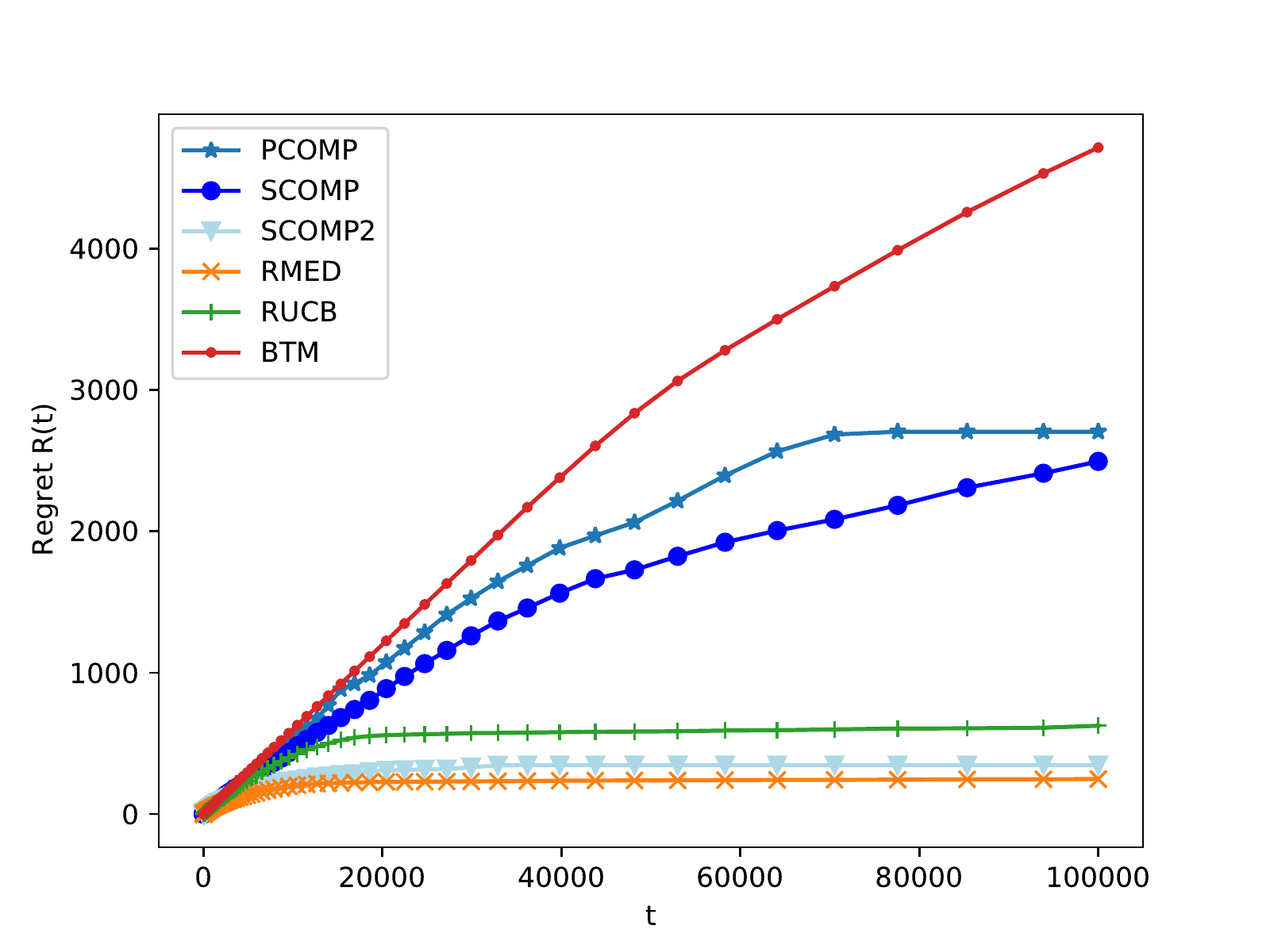}
         \caption{Six rankers}
     \end{subfigure}
     \begin{subfigure}[b]{0.45\textwidth}
         \centering
         \includegraphics[width=\textwidth]{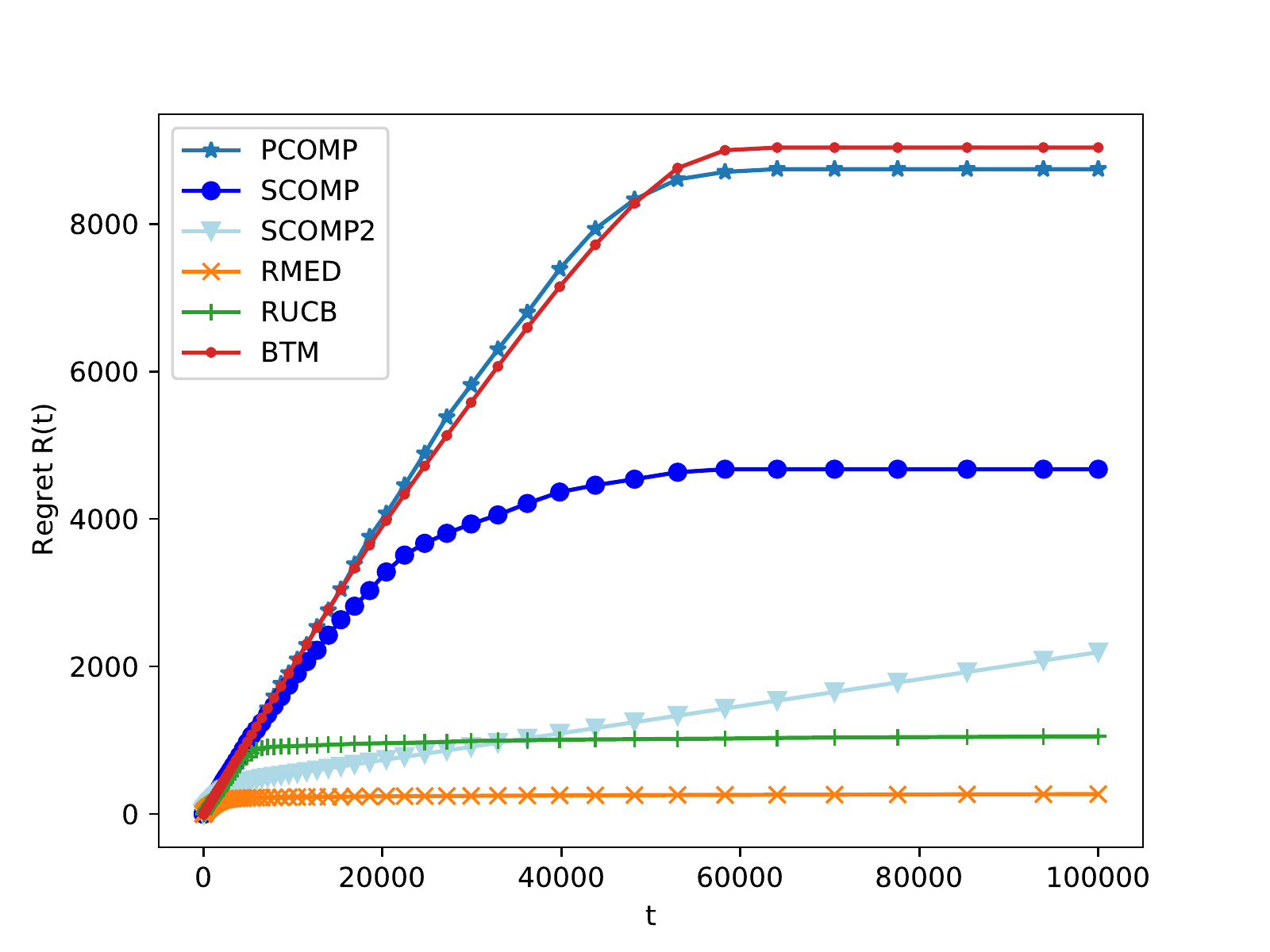}
         \caption{Sushi}
     \end{subfigure}
     \begin{subfigure}[b]{0.45\textwidth}
         \centering
         \includegraphics[width=\textwidth]{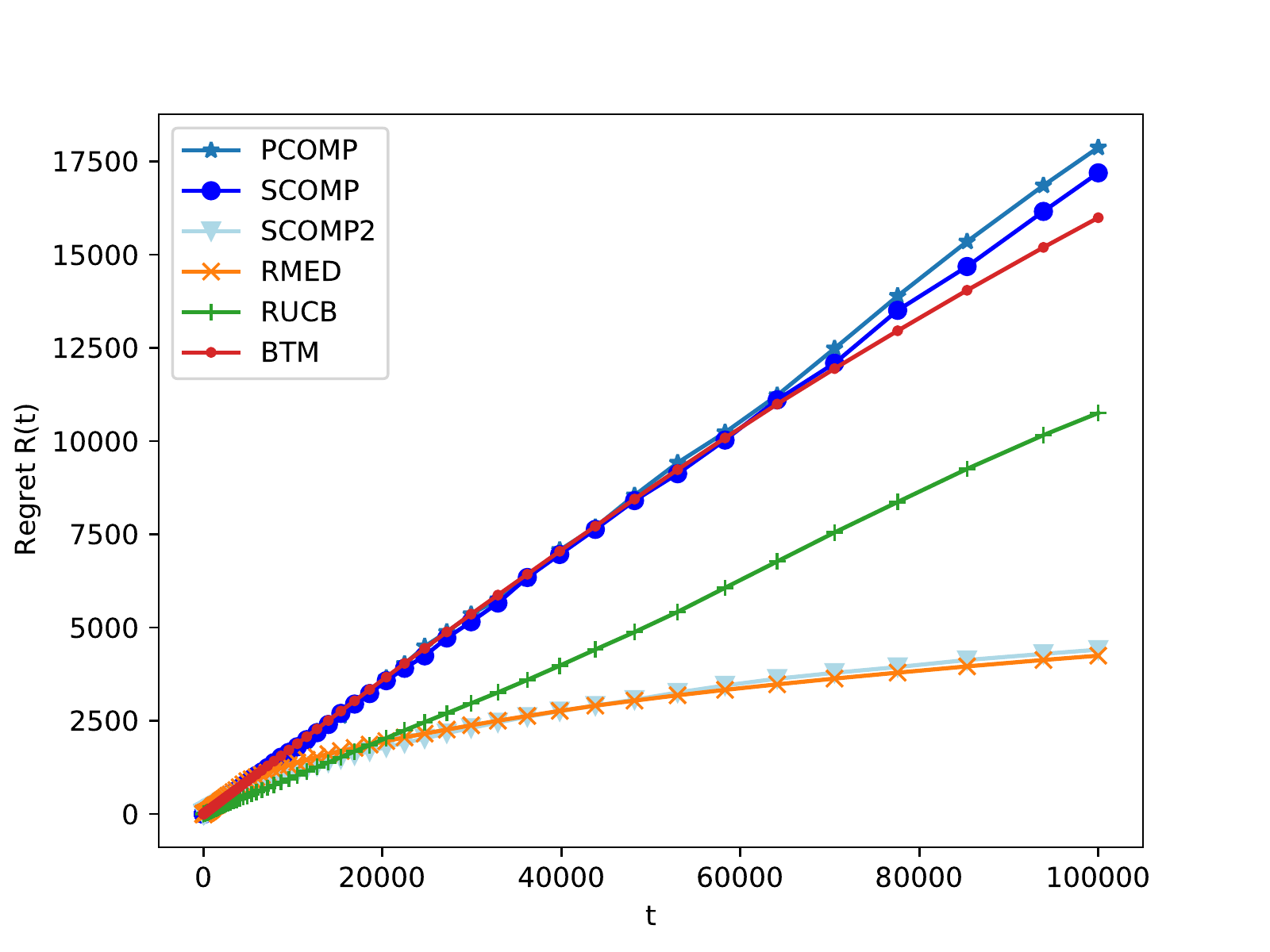}
         \caption{$\mathtt{SYN}$-$\mathtt{BTL}$}
     \end{subfigure}
     \begin{subfigure}[b]{0.45\textwidth}
         \centering
         \includegraphics[width=\textwidth]{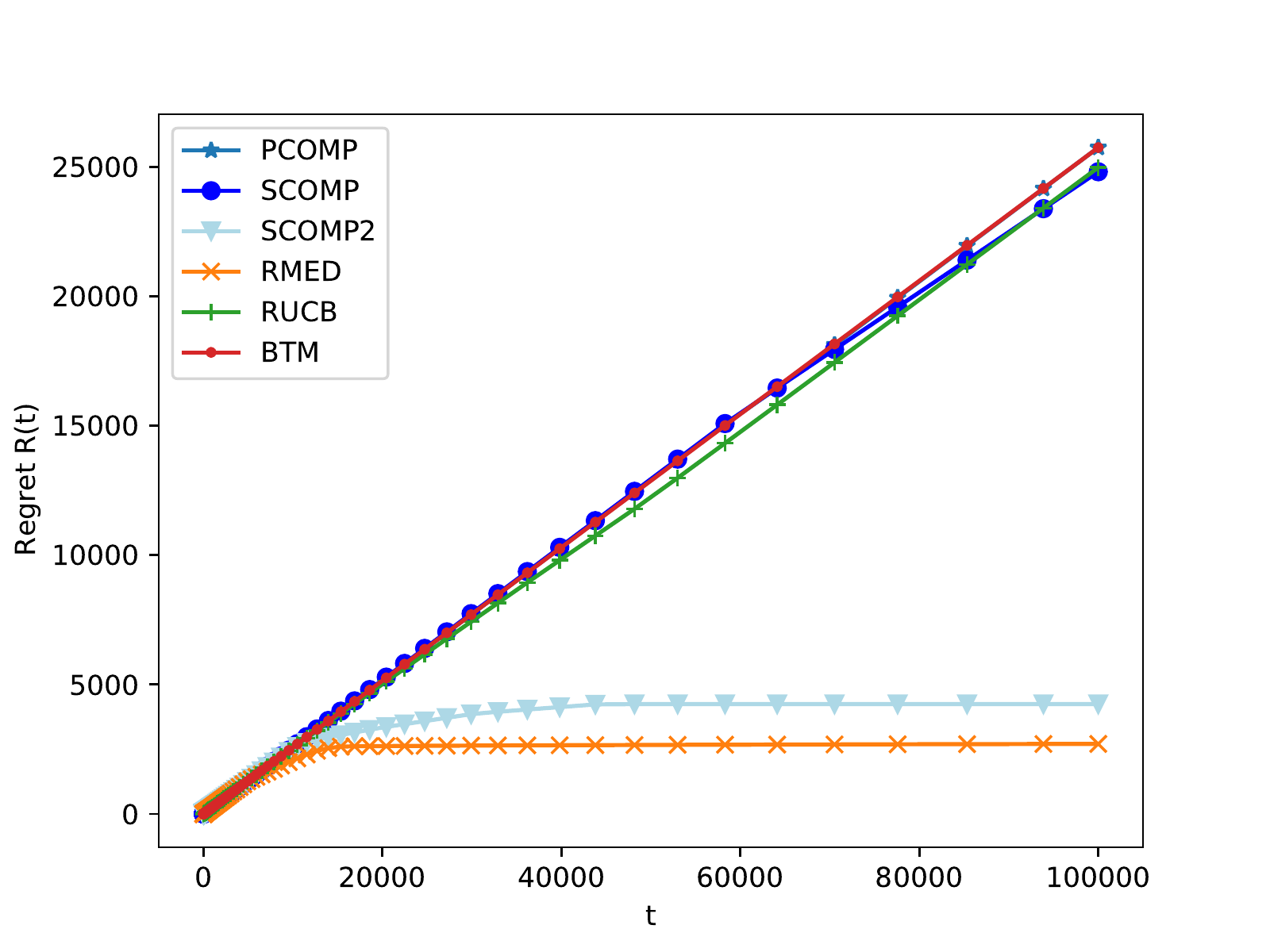}
         \caption{$\mathtt{SYN}$-$\mathtt{CD}$}
     \end{subfigure}
        \caption{Regret v/s t plots of algorithms}
        \label{fig:comparisons}
\end{figure}
\fi
As mentioned earlier, 
we compare our algorithms against a representative set of sequential dueling bandits algorithms
(RUCB \cite{Zoghi+14}, RMED1 \cite{Komiyama+15a}, and BTM \cite{YueJo11}). 
Note that the purpose of these experiments is to perform a sanity check to ensure that our batched algorithms, using a small number of batches,
perform well when compared with sequential 
algorithms.
We set $\alpha = 0.51$ for RUCB, and $f(K) = 0.3 K^{1.01}$ for RMED1, and $\gamma = 1.3$ for BTM. 
We chose these parameters as they are known to perform well both theoretically and empirically \cite{Komiyama+15a}.
We set $T = 10^5$, $\delta = 1/TK^2$ and $B = \lfloor\log(T)\rfloor = 16$.
We plot the results in \Cref{fig:comparisons}.
We observe that \scomp2 performs comparably to RMED1 in all datasets, even outperforms RUCB in $3$ out of the $4$ datasets, and always beats BTM. Notice that both \pcomp \ and \scomp \ considerably outperform BTM on the six rankers and sushi data; however their performance degrades on the synthetic data demonstrating the dependence on $K$.

\textbf{Trade-off with number of batches $B$.}
\ifConfVersion
\else
\begin{figure}[t]
     \centering
     \begin{subfigure}[b]{0.45\textwidth}
         \centering
         \includegraphics[width=\textwidth]{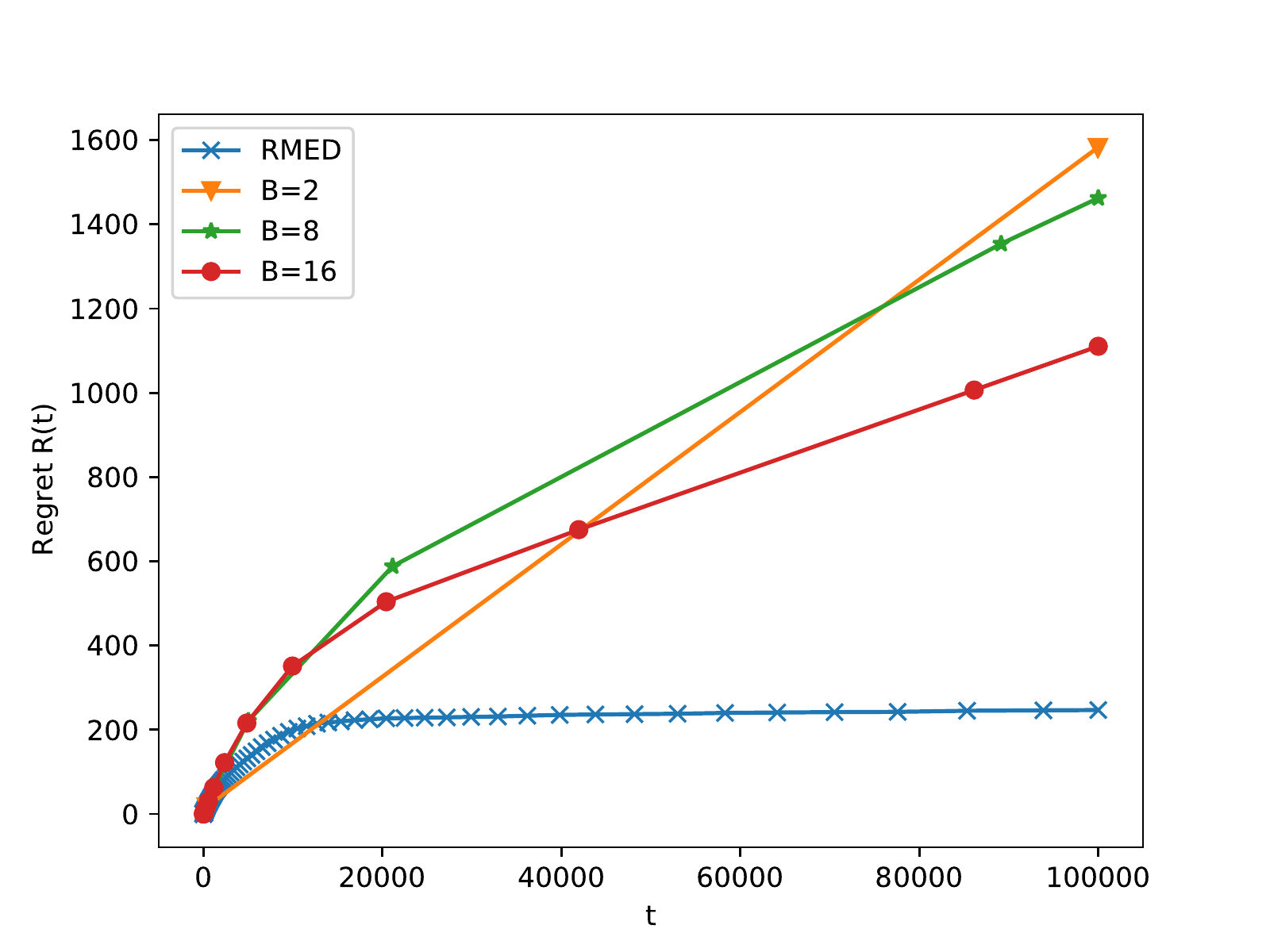}
         \caption{Six rankers}
     \end{subfigure}
     \begin{subfigure}[b]{0.45\textwidth}
         \centering
         \includegraphics[width=\textwidth]{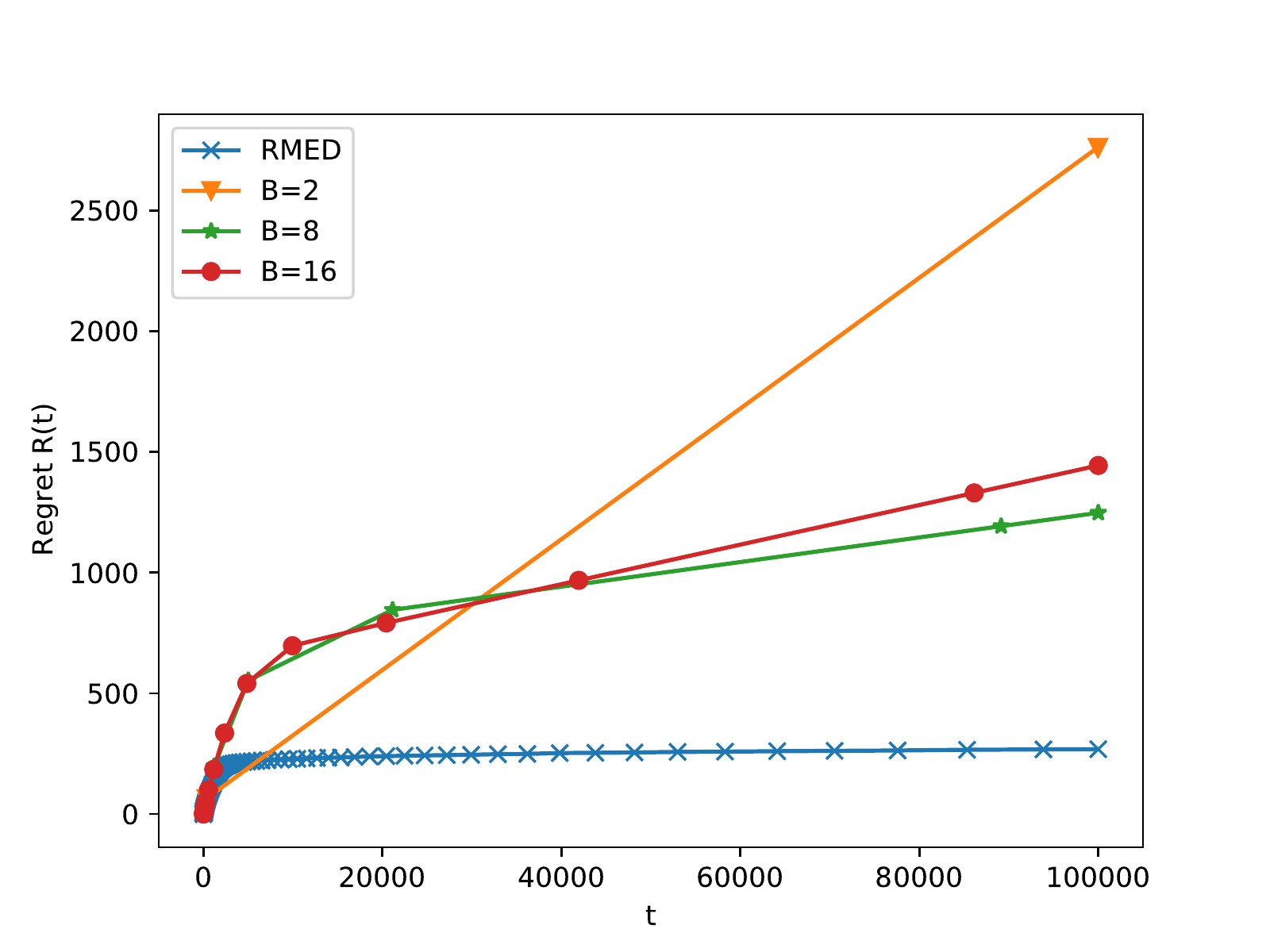}
         \caption{Sushi}
     \end{subfigure}
     \begin{subfigure}[b]{0.45\textwidth}
         \centering
         \includegraphics[width=\textwidth]{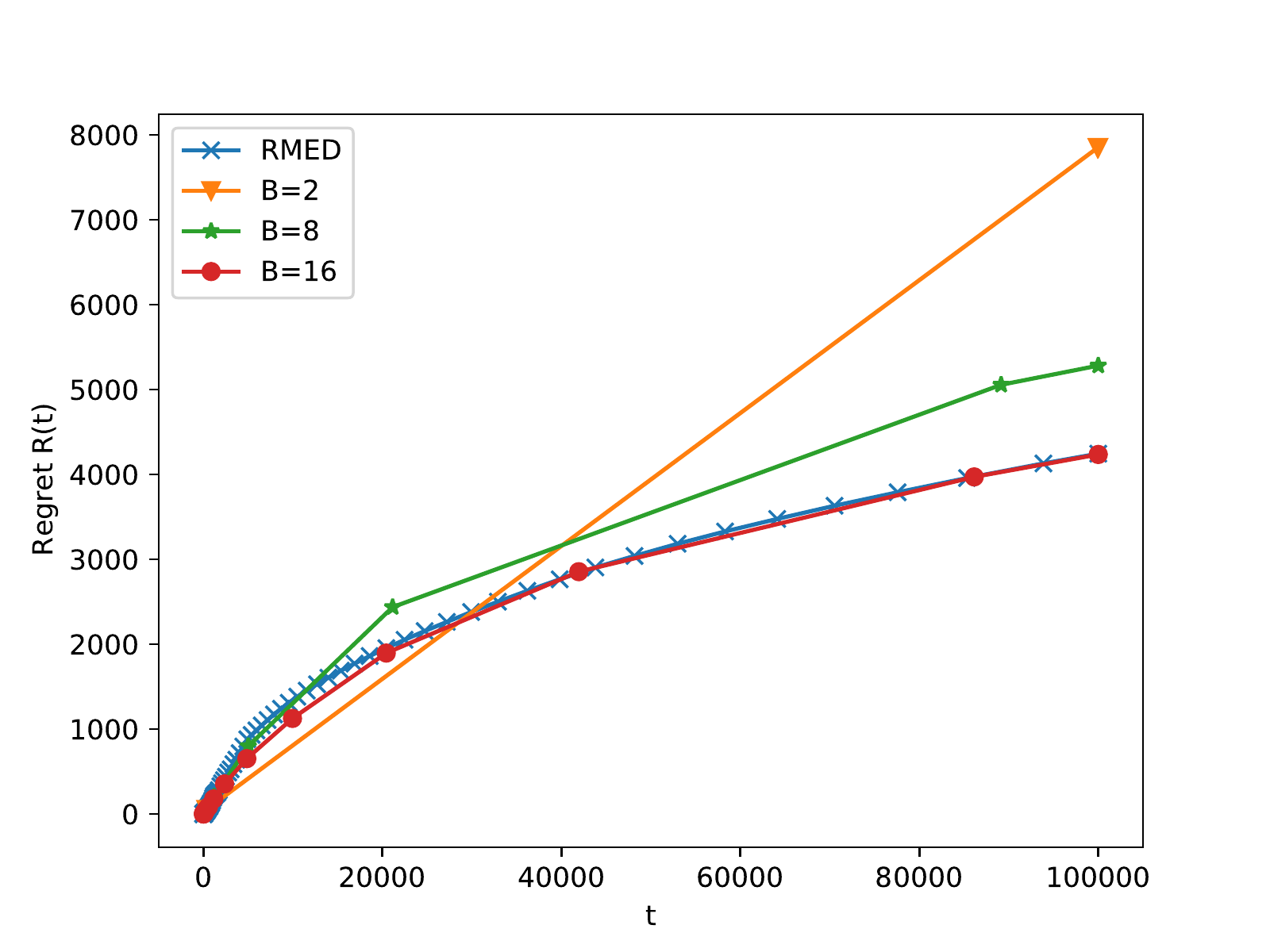}
         \caption{$\mathtt{SYN}$-$\mathtt{BTL}$}
     \end{subfigure}
     \begin{subfigure}[b]{0.45\textwidth}
         \centering
         \includegraphics[width=\textwidth]{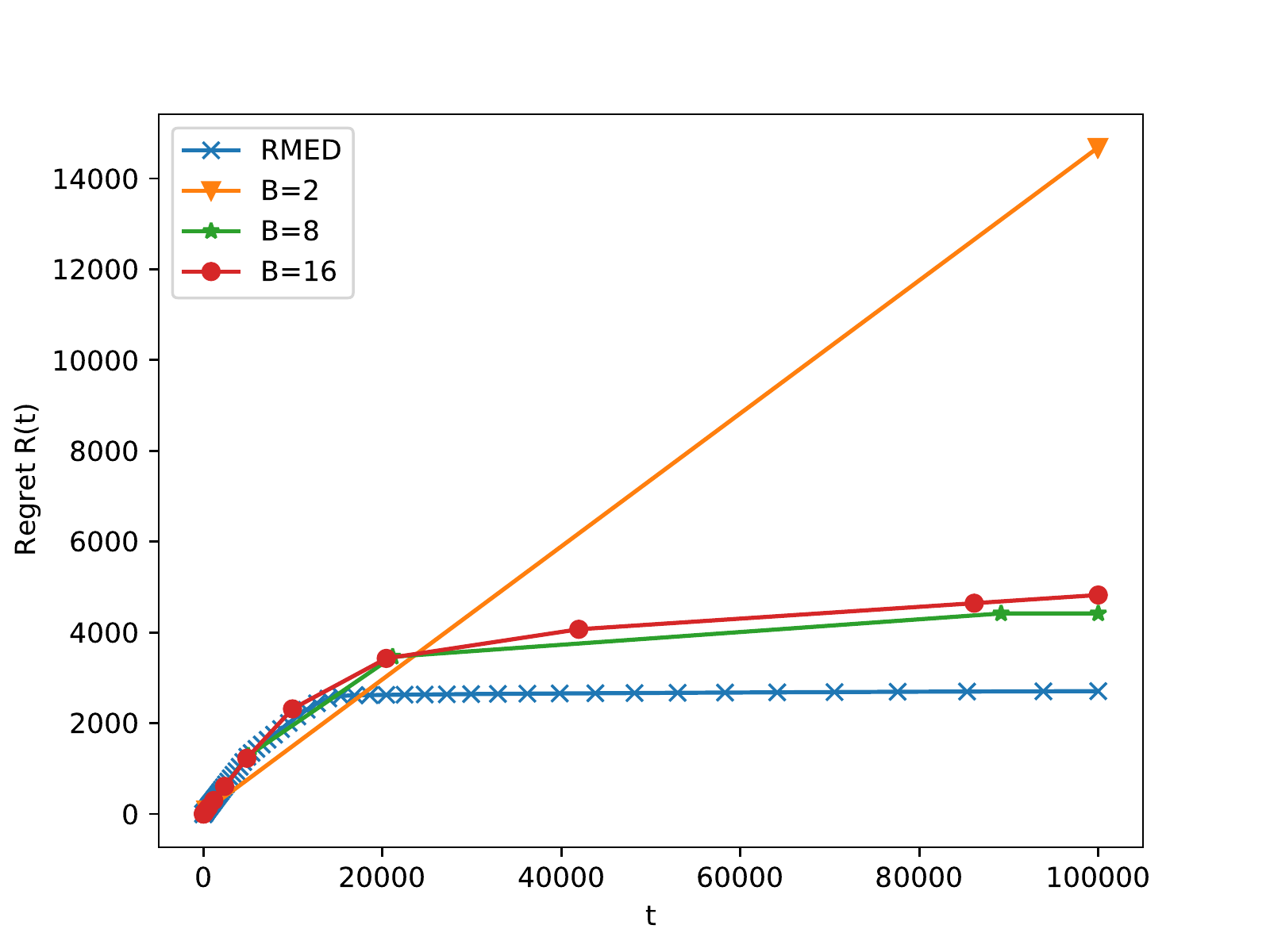}
         \caption{$\mathtt{SYN}$-$\mathtt{CD}$}
     \end{subfigure}
        \caption{Regret v/s B for \scomp2.}
        \label{fig:tradeoff}
\end{figure}
\fi
\ifConfVersion
\else
We study the trade-off of cumulative regret against the number of batches using \scomp2. We set $T=10^5$, and vary $B \in \{2, 8, 16\}$. We also plot the regret incurred by RMED1 as it performs the best amongst all sequential algorithms (and thus serves as a good benchmark). We plot the results in \Cref{fig:tradeoff}.
We observe that as we increase the number of batches, the (expected) cumulative regret decreases. Furthermore, we observe that on the synthetic datasets (where $K = 100$), the regret of \scomp2 approaches that of RMED1; in fact, the regret incurred is almost identical for the $\mathtt{SYN}$-$\mathtt{BTL}$ dataset.
\fi

{
\bibliographystyle{abbrv}
\bibliography{main}
}

\end{document}